%% file: main.tex
\documentclass{article}

\usepackage[final,nonatbib]{neurips_2022}

\usepackage[utf8]{inputenc} 
\usepackage[T1]{fontenc}    
\usepackage{microtype}

\usepackage{graphicx}
\usepackage{booktabs} 

\usepackage{hyperref}

\input{math_commands}
\usepackage{amsmath}
\usepackage{amssymb}
\usepackage{mathtools}
\usepackage{amsthm}

\usepackage{xcolor}
\usepackage{tabularx,longtable}
\usepackage{enumitem}
\usepackage{csquotes}
\usepackage{colortbl}

\usepackage{times}
\usepackage{latexsym}

\usepackage{dblfloatfix}
\usepackage{thm-restate}
\usepackage{caption}
\usepackage{subcaption}

\usepackage{geometry}
\usepackage{algorithm}
\usepackage[noend]{algpseudocode}

\usepackage[capitalize,noabbrev,nameinlink]{cleveref}

\theoremstyle{plain}
\newtheorem{theorem}{Theorem}[section]
\newtheorem{proposition}[theorem]{Proposition}
\newtheorem{lemma}[theorem]{Lemma}

\theoremstyle{definition}
\newtheorem{definition}[theorem]{Definition}

\theoremstyle{remark}

\usepackage[textsize=tiny]{todonotes}

\newcommand{\conv}{\operatorname{conv}}
\newcommand{\FFT}{\operatorname{FFT}}

\newcommand\aaa{\cellcolor{blue!50}}
\newcommand\bbb{\cellcolor{purple!50}}
\newcommand\ccc{\cellcolor{green!50}}
\newcommand\ddd{\cellcolor{cyan!50}}

\crefname{section}{Section}{Section}
\crefname{subsection}{Section}{Section}

\algdef{SE}[SUBALG]{Indent}{EndIndent}{}{\algorithmicend\ }%
\algtext*{Indent}
\algtext*{EndIndent}

\makeatletter
\algnewcommand{\LineComment}[1]{\Statex \hskip\ALG@thistlm  #1}

\title{projUNN: efficient method for training deep networks with unitary matrices}

\author{%
  Bobak T. Kiani\\
  MIT \\
\texttt{bkiani@mit.edu}
   \And
  Randall Balestriero \\
  Meta AI, FAIR \\
\texttt{rbalestriero@fb.com}
   \And
  Yann LeCun \\
  NYU \&
  Meta AI, FAIR \\
\texttt{yann@fb.com}
  \And
  Seth Lloyd \\
  MIT \&
  Turing Inc. \\
  \texttt{slloyd@mit.edu}
}

\begin{document}

\maketitle
\begin{abstract}
In learning with recurrent or very deep feed-forward networks, employing unitary matrices in each layer can be very effective at maintaining long-range stability. However, restricting network parameters to be unitary typically comes at the cost of expensive parameterizations or increased training runtime. We propose instead an efficient method based on rank-$k$ updates -- or their rank-$k$ approximation -- that maintains performance at a nearly optimal training runtime.
We introduce two variants of this method, named Direct (projUNN-D) and Tangent (projUNN-T) projected Unitary Neural Networks, that can parameterize full $N$-dimensional unitary or orthogonal matrices with a training runtime scaling as $O(kN^2)$. 
Our method either projects low-rank gradients onto the closest unitary matrix (projUNN-T) or transports unitary matrices in the direction of the low-rank gradient (projUNN-D). Even in the fastest setting ($k=1$), projUNN is able to train a model's unitary parameters to reach comparable performances against baseline implementations. In recurrent neural network settings, projUNN closely matches or exceeds benchmarked results from prior unitary neural networks. Finally, we preliminarily explore projUNN in training orthogonal convolutional neural networks, which are currently unable to outperform state of the art models but can potentially enhance stability and robustness at large depth. 
\end{abstract}

\section{Introduction}
Learning in neural networks can often be unstable when networks are very deep or inputs are long sequences of data \cite{arjovsky2016unitary,xiao2018dynamical}. For example, vanilla recurrent neural networks (RNNs) have recurrent states that are evolved via repeated application of a linear transformation followed by a pointwise nonlinearity, which can become unstable when eigenvalues of the linear transformation are not of magnitude one. Unitary matrices, which have eigenvalues of magnitude one, can naturally avoid this issue and have been used as a means to overcome these so-called vanishing and exploding gradients \cite{arjovsky2016unitary,jing2017tunable}. More recently, unitary convolutional layers have been similarly constructed to help build more stable deep networks that are norm-preserving in their transformations \cite{li2019preventing,sedghi2018singular}.

\begin{table*}[!t]
\caption{When training RNNs on inputs with sequence length $T$, \textsc{projUNN} achieves nearly optimal runtime complexity while maintaining full parameterization of the unitary manifold.
}
\centering
\small
\begin{tabular}{lccl}
\hline
Model                     & \begin{tabular}[c]{@{}c@{}}Complexity of\\ gradient step\end{tabular} & \begin{tabular}[c]{@{}c@{}}Layers to fully\\ parameterize$^a$\end{tabular} & Method of parameterization                 \\ \hline
EURNN (tunable, n layers) \cite{jing2017tunable} & $O(Tn^2)$                                                     & $O(n)$                                                                            & Sequence of rotations \\
oRNN (n layers) \cite{mhammedi2017efficient}         & $O(Tn^2)$                                                     & $O(n)$                                                                            & Sequence of householder reflections        \\
full-capacity URNN \cite{wisdom2016full}                   & $O(Tn^2 + n^3)^b$                                                     & 1                                                                               & Parameterized matrix entries        \\
expRNN \cite{lezcano2019cheap}                    & $O(Tn^2 + n^3)^b$                                                     & 1                                                                               & Parameterized matrix in Lie algebra        \\
\textsc{projUNN} (our method)                 & $O(Tn^2 + kn^2)^c$                                                     & 1                                                                               & Parameterized matrix entries               \\ \hline
\multicolumn{4}{l}{\begin{tabular}[c]{p{0.95\linewidth}}\small 
$^a$ layers needed to parameterize the full unitary space, $^b$ approximations exist which may reduce runtimes though these approximations are not implemented here and can significantly bias the gradient \cite{lezcano2019cheap}, $^c$ runtime shown for typical setting when $k \ll n$ where $k$ is the rank of gradient updates   \end{tabular}}
\end{tabular}
\label{tab:model_comps}
\end{table*}

In the RNN setting, prior algorithms to apply $n \times n$ unitary matrices in RNNs have parameterized matrices into layers of unitary or orthogonal transformations or parameterized the Lie algebra of the unitary or orthogonal group (see \cref{tab:model_comps}). In the layer-wise setting, unitarity is enforced for all values of parameters, but many layers are required to form a composition that can recreate any desired unitary, \textit{i.e.,} fully parameterizing an $n \times n$ unitary requires $O(n)$ layers. By parameterizing the Lie algebra \cite{lezcano2019cheap,hyland2017learning}, algorithms perform better on common benchmarks but have the drawback that performing gradient optimization on an $n \times n$ unitary requires $O(n^3)$ operations generically per step. Though not an issue with the small to medium sized models used today, this $O(n^3)$ is still $O(n)$ slower than standard methods of forward- and back-propogation in RNNs. 

Motivated by the feature that gradients in neural networks are typically approximately low rank, we show that gradient updates to unitary/orthogonal matrices can be efficiently performed in low rank settings. We propose a new model called \textsc{projUNN} where matrices are first updated via gradient based optimization and then projected back onto the closest unitary (\textsc{projUNN-D}) or transported in the direction of the gradient (\textsc{projUNN-T}). \textsc{projUNN} has near-optimal runtime complexity unlike other existing algorithms for unitary RNNs (\cref{tab:model_comps}) and is especially effective even in the most extreme case where gradients are approximated by rank one matrices. In RNN learning tasks, \textsc{projUNN} matches or exceeds benchmarks of state-of-the-art unitary neural network algorithms. 

Though we present our model first in the RNN setting, we show that there is a direct extension of \textsc{projUNN} to the case of orthogonal/unitary convolution which we explore further. Here, we perform unitary/orthogonal convolution in the Fourier domain as inspired by \cite{trockman2021orthogonalizing}. Our algorithm runs efficiently in the convolutional setting especially for filters of large size and many channels (see \cref{app:runtime} for more details). 

\section{Related works}
Maintaining stability in neural networks via orthogonal or unitary matrices has a rich history of study in machine learning, both from an applied and theoretical perspective. Here, we briefly mention the most related works and algorithms we use in comparison to our \textsc{projUNN}. For a more holistic review of prior work in unitary neural networks and other related topics, please see \cref{app:prior_works}.   

Unitary neural networks were first designed to address the issue of vanishing and exploding gradients in RNNs while learning information in very long sequences of data more efficiently than existing parameterizations such as the long-short term memory unit (LSTM) \cite{hochreiter1997long}. Early algorithms \cite{arjovsky2016unitary,mhammedi2017efficient} maintained unitarity by constructing a series of parameterized unitary transformations. Perhaps the most effective of these methods is the efficient unitary recurrent neural network (EUNN) \cite{jing2017tunable} which parameterized unitary matrices by composing layers of Givens rotations, Fourier transforms, and other unitary transformations. The unitary RNN (uRNN) of \cite{wisdom2016full} and the Cayley parameterization (scoRNN) of \cite{helfrich2018orthogonal} parameterized the full unitary space and maintained unitarity by performing a Cayley transformation. Later, \cite{lezcano2019cheap} introduced the exponential RNN (expRNN) which parameterized unitary matrices in the Lie algebra of the orthogonal/unitary group. Though the uRNN, scoRNN, and expRNN perform well on benchmarks, their algorithms require matrix inversion or SVD steps which are time-consuming in high dimensions.

For convolutional neural networks, \cite{sedghi2018singular} showed how to efficiently calculate the singular values of a linear convolution and proposed an algorithm for projecting convolutions onto an operator-norm ball which relied on a series of costly projection steps. \cite{li2019preventing} introduced a block convolutional orthogonal parameterization (BCOP) which was faster and more efficient than the methods in \cite{sedghi2018singular}, but required extra parameters in its parameterization and only parameterized a subset of the space of orthogonal convolutions. Most recently, \cite{singla2021skew} implemented orthogonal convolutions by parameterizing the Lie algebra of the orthogonal group via their skew orthogonal convolution (SOC) algorithm which approximates orthogonal convolutions especially well for small filter sizes. Finally, \cite{trockman2021orthogonalizing} performs convolutions in the Fourier domain via application of the Cayley transform. Our orthogonal/unitary convolutional parameterization is inspired by their approach and improves their runtime for convolutions over many channels.


\section{Notation and background}
Vectors and matrices are denoted with bold lower-case and upper-case script, $\vv$ and $\mV$, respectively. Scalars are denoted by regular script $e$ and tensors are denoted by bold text $\tT$. The complex conjugate of a complex-valued input $\cdot$ is denoted by $\cdot^*$ (ignored when real-valued). The transpose of a matrix $\mM$ is denoted by $\mM^\intercal$ and the conjugate transpose of a matrix is denoted by $\mM^\dagger$. We denote the Frobenius norm of a matrix by $\| \cdot \|_F$ and the spectral norm of a matrix by $\| \cdot \|_2$. 

Here, we provide a brief overview of the unitary/orthogonal groups and refer readers to \cref{app:math_background} for a more detailed mathematical background. The set of $n \times n$ orthogonal $O(n)$ and unitary $U(n)$ matrices are both Lie groups defined as
\begin{align}
    O(n) = \left\{ \mM \in \mathbb{R}^{n \times n} | \mM \mM^\intercal = \mI  \right\}, \; \; \; \; \; \;
    U(n) = \left\{ \mM \in \mathbb{C}^{n \times n} | \mM \mM^\dagger = \mI  \right\}.
\end{align}
Constraining matrices in $O(n)$ and $U(n)$ to have determinant equal to one constructs the special orthogonal $SO(n)$ and unitary $SU(n)$ groups respectively. The Lie algebra or tangent space of the identity of $O(n)$ and $U(n)$ are the set of skew symmetric $\mathfrak{o}(n)$ and skew Hermitian $\mathfrak{u}(n)$ matrices,
\begin{align}
    \mathfrak{o}(n) = \left\{\boldsymbol{A} \in \mathbb{R}^{n \times n}: \mA + \mA^\intercal = 0 \right\}, \; \; \; \; \; \;
    \mathfrak{u}(n) = \left\{\boldsymbol{A} \in \mathbb{C}^{n \times n}: \mA + \mA^\dagger = 0 \right\}.
\end{align}

The matrix exponential $\exp(\cdot)$ is a map from the Lie algebra to the associated Lie group. The map is surjective if the Lie group is compact and connected -- a property which holds for the unitary and special orthogonal groups but not the orthogonal group. 

\section{Projected unitary networks}

\begin{figure}[t]
 \captionsetup[subfigure]{aboveskip=1pt,belowskip=-2pt}  
 \begin{subfigure}{0.5\textwidth}
    \centering
    \includegraphics[width=\linewidth]{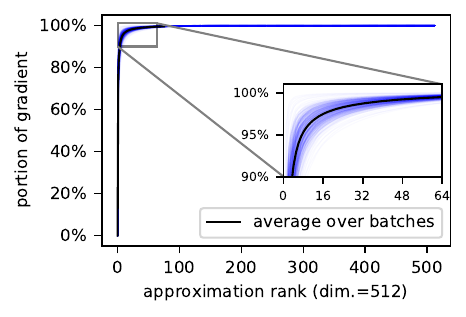}
    \caption{Low rank approximations}
    \label{fig:cifar_low_rank}
  \end{subfigure}%
  \hspace*{\fill}   
 \begin{subfigure}{0.5\textwidth}
    \centering
    \includegraphics[width=\linewidth]{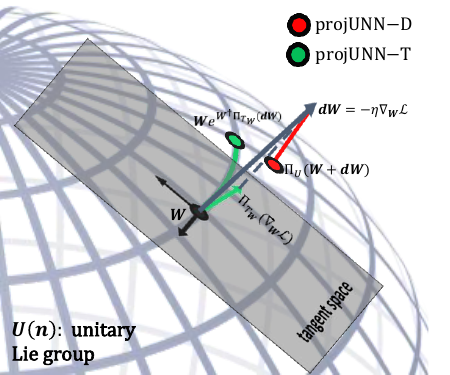}
    \caption{Illustration of algorithm}
    \label{fig:projection}
  \end{subfigure}%
  \hspace*{\fill}   
  \caption{(a) Low rank approximations capture most of the Frobenius norm of the gradient of a $512 \times 512$ matrix in the convolution filter (512 channels) of the last residual block of Resnet-9. Blue lines plot gradients of a single batch during training of our \textsc{projUNN} algorithm on CIFAR10 over a single epoch (see \cref{sec:numerical_rank} for details and equivalent plot for RNN architecture). (b) Illustration of a single gradient update via gradient descent with learning rate $\eta$. \textsc{projUNN-D} (pictured in red) directly projects the gradient update back onto the unitary/orthogonal manifold. \textsc{projUNN-T} (pictured in green) first projects onto the tangent space (Lie algebra) and then performs a rotation in that direction via the exponential map.}
\end{figure}

Our \textsc{projUNN} algorithm is motivated by the simple observation that most of the ``information" of a typical gradient in a deep learning task is captured in a low rank subspace of the complete gradient. \cref{fig:mnist_low_rank} illustrates this feature when training our \textsc{projUNN} convolutional network on CIFAR10. We include further analysis and justification of this low rank behavior in \cref{sec:numerical_rank}. As we will show, we can perform updates on the low rank subspace of the gradient efficiently by approximating the gradient with a low rank matrix and performing projections of parameters onto that low rank subspace. Our experiments show that this methodology, even with rank one approximations, is effective at learning and empirically introduces a form of ``beneficial" stochasticity during gradient descent.


Based on how the projection is performed, our \textsc{projUNN} algorithm takes two forms illustrated in \cref{fig:projection}. The directly projected unitary neural network (\textsc{projUNN-D}) projects an update onto the closest unitary/orthogonal matrix in Frobenius norm. The tangent projected unitary neural network (\textsc{projUNN-T}) projects gradients onto the tangent space and transports parameters in that direction.


\subsection{\textsc{projUNN-D}} \label{sec:projunnd_alg}
\textsc{projUNN-D} takes advantage of the fact that the polar transformation returns the closest unitary or orthogonal matrix in the Frobenius norm to a given matrix (not necessarily unitary or orthogonal):
\begin{restatable}[Projection onto unitary manifold \cite{keller1975closest}]{lemma}{closestUnitary}
\label{lemma:DirectProjection}
Given a matrix $\mA \in \mathbb{C}^{n \times n}$:
\begin{equation}
    \Pi_{U}(\mA) = \argmin_{\mU \in \, \mathcal{U}(n)} \| \mA - \mU \|_F^2 = \mA (\mA^\dagger \mA)^{-\frac{1}{2}},
\end{equation}
where $\mathcal{U}(n)$ indicates the set of $n \times n$ unitary matrices.
\end{restatable}

Note, that if the matrix $\mA$ is real, then the projection above will be onto an orthogonal matrix. Given \Cref{lemma:DirectProjection}, \textsc{projUNN-D} performs optimization in two steps, which are illustrated in \cref{fig:projection}. First, matrix entries are updated via a standard learning step as in gradient descent, constructing a new matrix that is generally no longer unitary. In the second step, \textsc{projUNN-D} returns the unitary or orthogonal matrix closest in the Frobenius norm to the inputted matrix using \cref{lemma:DirectProjection}. At first sight, the second step would require $O(n^3)$ time to perform, but we can take advantage of the fact that gradient updates are typically approximately low rank (see \cref{sec:numerical_rank}). Efficient low rank approximations can be obtained using sampling methods detailed in \cref{sec:sampling}. With this in mind, we show that rank $k$ updates can be performed in $O(kn^2)$ time when $k \ll n$.

\begin{restatable}[Low rank unitary projection]{theorem}{DirectProjection}
\label{thm:DirectProjection}
Let $\mU$ be an $n \times n$ orthogonal/unitary matrix perturbed by $\mG_k$, a rank $k$ matrix. Then the projection onto the closest orthogonal/unitary matrix defined below can be performed in $O(k(n^2 + nk + k^2))$ steps.
\begin{equation}
\label{eq:direct_update}
    \mU + \mG_k \to \argmin_{\mV \in \, \mathcal{U}} \| \mU + \mG_k - \mV \|_F^2.
\end{equation}
\end{restatable}
To achieve this runtime, we perform updates completely in an $O(k)$ subspace of the full vector space. The operation $(\mU + \mG_k)[(\mU + \mG_k)^\dagger (\mU + \mG_k)]^{-1/2}$ can be decomposed into a series of $O(k)$ matrix-vector operations and an eigendecomposition of a $2k \times 2k$ sub-matrix. The complete proof and details are deferred to \cref{app:deferred_proofs}. One limitation of the above is that the eigendecomposition and inversion of a low rank matrix can cause numerical instability after many update steps. We discuss this further in \cref{app:numerical_stability} where we also provide options to alleviate this instability. \textsc{projUNN-T}, which we discuss next, does not require matrix inversion and is thus empirically more stable.

\subsection{\textsc{projUNN-T}}
\textsc{projUNN-T} maintains unitarity of matrices by orthogonally projecting gradient updates onto the tangent space and then performing a rotation in the direction of the projection (\textit{i.e.,} along the geodesic). As in \textsc{projUNN-D}, there is a closed form for the orthogonal projection:
\begin{restatable}[Tangent space projection \cite{wisdom2016full}]{lemma}{orthogonalProjectionTangent}
\label{lemma:TangentProjection}
Given the tangent space $T_{\mU} U(n)$ of an orthogonal/unitary matrix $\mU$, the orthogonal projection $\Pi_{T_\mU}$ with respect to the canonical metric $\langle \mX, \mY \rangle =\operatorname{Re} \left( \Tr [\mX^\dagger \mY] \right)$ is
\begin{equation}
    \Pi_{T_\mU}(\mX) = \frac{1}{2}\left(\mX - \mU \mX^\dagger \mU \right).
\end{equation}
Similar to \cref{lemma:DirectProjection}, this projection also returns the closest matrix in Frobenius norm to $\mX$ in the tangent space,
\begin{equation}
    \min_{\mY \in T_{\mU}U(n) } \left\| \mY - \mX \right\|_F = \Pi_{T_\mU}(\mX).
\end{equation}
\end{restatable}

Similar to \textsc{projUNN-D}, \textsc{projUNN-T} performs learning in two steps. First, a gradient update $\mG$ is projected onto the tangent space using \Cref{lemma:TangentProjection}. Then, the orthogonal/unitary matrix is transported or rotated in the direction of the projection by application of the exponential map via the update rule \cite{lezcano2019cheap,wisdom2016full},  
\begin{equation}
\label{eq:algebra_update_rule}
    \mU \to \mU \exp\left[ - \eta \mU^\dagger \Pi_{T_\mU}(\mG) \right],
\end{equation}
where $\eta$ denotes the learning rate. This update rule is an example of Riemannian gradient descent where we use the exponential map to transport gradient updates along the unitary/orthogonal manifold \cite{bonnabel2013stochastic}. Here, we transport the matrix $\mU$ along the geodesic in the direction of $\Pi_{T_\mU}(\mG)$. This can be related to the update of \textsc{projUNN-D} which is an example of a retraction or an approximation to the exponential map of \textsc{projUNN-T} (see \cref{app:first_order_equivalence}).

The update rule above requires matrix exponentiation and multiplication, both costly steps which can be sped up when $\mG$ is a low rank matrix. Namely, to perform a rank $k$ gradient update, we obtain an equivalent runtime scaling of $O(kn^2)$ for the \textsc{projUNN-D} when $k \ll n$.


\begin{restatable}[Low rank tangent transport]{theorem}{AlgebraProjection}
\label{thm:AlgebraProjection}
Let $\mU$ be an $n \times n$ orthogonal/unitary matrix perturbed by $\mG_k$, a rank $k$ matrix. Then projecting $\mG_k$ onto the tangent space and performing a rotation in that direction as defined in \cref{eq:algebra_update_rule} can be performed in $O(k(n^2 + nk + k^2))$ steps.
\end{restatable}
As with the \textsc{ProjUNN-D}, we achieve this runtime by performing the update above completely in an $O(k)$ subspace of the full vector space. The update via the exponential map can similarly be decomposed into a series of $O(k)$ matrix-vector operations and an eigendecomposition of a $2k \times 2k$ sub-matrix. Proper manipulations of the eigenvalues of the sub-matrix implement updates via the exponential map. The complete proof and details are deferred to \cref{app:deferred_proofs}.

\subsection{Sampling methods}
\label{sec:sampling}

Commonly, gradients can have large rank but have still have many small singular values (\textit{e.g.,} see \cref{fig:cifar_low_rank}). Here, a matrix $\mA$ is deemed approximately low rank (see more details in \cref{sec:numerical_rank}), and one can obtain a rank $k$ approximation $\mA_k$ of $\mA$ by sampling from rows and columns of $\mA$. We use two sampling algorithms.
The \textbf{LSI sampling} algorithm \cite{papadimitriou2000latent} obtains a rank $k$ approximation to an $n \times n$ matrix $\mA$ in time $O(kn^2 \log n)$. The algorithm projects the matrix $\mA$ onto a random orthogonal subspace and then applies SVD based methods to the projected matrix to obtain the low rank approximation to that matrix. This algorithm features low approximation errors even for small $k$ and is used extensively in our implementation. The \textbf{column sampling} (linear time SVD) algorithm \cite{drineas2006fast} samples from the columns of an $n \times n$ matrix $\mA$ to obtain a rank $k$ approximation in $O(c^2n + c^3)$ time, where $c$ is a hyperparameter indicating the number of columns sampled. Typically, $c$ is chosen as a multiple of $k$ so the runtime is $O(k^2n + k^3)$. In implementing this algorithm, we calculate the right singular vectors via matrix multiplication of the left singular vectors so the total runtime is $O(kn^2 + k^2n + k^3)$. 

We note that the two procedures described above, though sufficient for our purposes, can be further optimized in their asymptotic runtime. For sake of completeness, we discuss two of these other sampling algorithms in \cref{sec:numerical_rank}.


\subsection{Extension to unitary or orthogonal convolution}
\label{sec:projunn_conv}
Unitary/orthogonal convolutions are linear convolution operations that also preserve the $2$-norm (isometric). Restricting convolutions to be unitary/orthogonal typically results in a drop in performance on standard imaging tasks when used in isolation, but prior work has explored unitary/orthogonal convolutions to potentially improve algorithmic stability and robustness (see \cref{app:convolution_unitary} for more background) \cite{li2019preventing,trockman2021orthogonalizing}. We describe here how \textsc{projUNN} can be used to implement unitary/orthogonal convolutions in potentially a more efficient manner. 

Given input tensor $\tX \in \mathbb{C}^{ M \times N \times C}$ where $C$ is the number of channels of an $M \times N$ input, linear convolution (or technically cross-correlation) with a filter $\tW \in \mathbb{C}^{M \times N \times C \times C}$ is defined as
\begin{equation}
   \left[\conv_{\tW}(\tX) \right]_{p,q,d} = \sum_{c = 1}^{C} \sum_{m=1}^{M} \sum_{n=1}^N \tW_{m,n,d,c} \tX_{p+m,q+n,c} ,
\end{equation}
where the indexing above is assumed to be cyclic (taken modulus the corresponding dimension) \cite{lecun1995convolutional,goodfellow2016deep}. Orthogonal/unitary convolutions form a subset of filters that preserve norms, \textit{i.e.,} filters $\tW$ such that $\|\conv_{\tW}(\tX)\| = \|\tX\| $. Equivalently, $\conv_{\tW}(\cdot)$ is orthogonal/unitary if the Jacobian of the transformation is also orthogonal/unitary. To maintain unitarity/orthogonality, we set the dimensions of the filter $\tW$ above such that it returns an output $\tY$ of the same dimension as the input $\tX$. One can also perform semi-orthogonal or semi-unitary convolution by appropriately zero-padding an input or truncating from dimensions in the output.

Standard convolutional filters are typically supported over a sparse set of local elements, but performing orthogonal/unitary convolution generally requires implementing convolutions with filters supported over all elements resulting in slower runtimes. One can locally parameterize convolutional filters in the Lie algebra of the orthogonal/unitary group; nevertheless the exponential map into the Lie group expands the support of the filter:
\begin{equation}
    \exp[\conv_{\tL}] (\tX) = \tX + \tL * \tX + \frac{1}{2} \tL *^2 \tX  + \frac{1}{6} \tL *^3 \tX + \cdots
\end{equation}

Thus, enforcing unitarity in convolutions generally requires additional overhead over the traditional setting of locally supported filters, but by performing convolution in the Fourier domain, runtimes for full-width filters can be optimally improved to $O(N^2 C \log(N) +N^2C^2)$ \cite{mathieu2013fast}:
\begin{equation}
    \left[\FFT \conv_{\tW}(\tX) \right]_{\widehat{r},\widehat{s},:} = \widehat{\tW}^*_{\widehat{r},\widehat{s},:,:} \; \left[\FFT\tX \right]_{\widehat{r},\widehat{s},:} \;,
    \label{eq:block_diag_conv}
\end{equation}
where $\widehat{\tW}_{i,j,:,:}$ is the value of the $\widehat{r}$ and $\widehat{s}$ frequency of $\tW$ across all channels in the Fourier domain and $\FFT$ is the 2-dimensional fast Fourier transformation.

Our method is inspired by that of \cite{trockman2021orthogonalizing} 
which transformed $\tW$ into Fourier space and performed a Cayley transformation (approximation to the exponential map into the Lie group) over the matrices indexed by $\widehat{\tW}_{\widehat{r},\widehat{s},:,:}$ which requires $O(N^2C^2 \log(N) + N^2C^3) $ operations. For our algorithm, we parameterize $\tW$ in the Fourier domain and only manipulate $\widehat{\tW}$ (see \cref{app:convolution_unitary} for a depiction of our parameterization). By parameterizing $\widehat{\tW}$ directly and performing rank $k$ updates using our \textsc{projUNN}, this runtime can be improved to $O(N^2C \log(N) + kN^2C^2)$ which is optimal when $k \ll N$. Our procedure for performing unitary/orthogonal convolution on an input $\tX$ with filter $\tW$ essentially follows the steps in \cref{eq:block_diag_conv}: perform an $\FFT$ on $\tX$, block-multiply this by $\widehat{\tW}$, and perform an inverse $\FFT$ on the output to obtain the final result. 

\paragraph{Limitations} Unitary/orthogonal convolutions are implemented in a cyclic fashion (\textit{i.e.,} indices are taken modulus the dimension) which is not the standard approach but has been used before to accelerate convolutional operations \cite{mathieu2013fast}. Additionally, we parameterize convolution filters to have support over all possible elements (full-width), which can be expensive in memory. One can restrict the convolution to local terms in the Lie algebra, but this would not improve runtime as our algorithm runs in the Fourier space. To target local terms in a convolution, we instead propose for future work to implement a regularizer which has a specified support and penalizes the norm of the filter outside that support. Finally, the space of orthogonal convolutions has multiple disconnected components, which can present challenges for gradient based learning \cite{li2019preventing}. However, we can avoid this drawback by implementing \textsc{projUNN} using fully supported filters in the space of unitary convolutions which is connected (proof deferred to \cref{app:unitary_connected}). 

\begin{restatable}[Unitary convolutional manifold is connected]{theorem}{ConnectedConvolution}
\label{thm:unitary_conv_connected}
The space of unitary convolutions with filters of full support has a single connected component.
\end{restatable}

\begin{algorithm}
\caption{\textsc{projUNN} update step}
\small
\begin{algorithmic}[1]
\Require unitary matrix $\mU \in \mathbb{C}^{N \times N}$ or orthogonal matrix $\mU \in \mathbb{R}^{N \times N}$
\Require gradient update $\Delta \mU \in \mathbb{C}^{N \times N}$ or $\Delta \mU \in \mathbb{R}^{N \times N}$
\Require hyperparameter $k$ corresponding to rank of approximation
\State Obtain rank $k$ approximation to $\Delta \mU$ with output $\sum_{i=1}^k \va_i \vb_i^\dagger \approx \Delta \mU$ (see \Cref{sec:sampling})
\State  Follow steps in \Cref{thm:DirectProjection} (\textsc{projUNN-D}) or \Cref{thm:AlgebraProjection} (\textsc{projUNN-T}) in \Cref{app:deferred_proofs}:
    \Indent
    \State Perform Gram-Schmidt (via QR decomposition) on concatenation of vectors $\mU^\dagger \va_i$ and $\vb_i$ for all $i \in [k]$:
        \Indent
        \LineComment{$\;\;\;\;$ output $\mQ \in \mathbb{C}^{N \times k}$ as semi-orthogonal matrix containing basis after Gram-Schmidt}
        \EndIndent
    \State Form matrix $\mK \in \mathbb{C}^{2k \times 2k}$ below:
        \Indent
        \LineComment{$\;\;\;\;$ \textsc{projUNN-D}: $\mK = \sum_{i=1}^k \mQ^\dagger \mU^\dagger \va_i \vb_i^\dagger \mQ +  \mQ^\dagger \vb_i \va_i^\dagger \mU \mQ + \sum_{i=1}^k \sum_{j=1}^k (\va_i^\dagger \va_j) \mQ^\dagger \vb_i \vb_j^\dagger \mQ $}
        \LineComment{$\;\;\;\;$$\;\;\;\;$ \textit{see \Cref{eq:direct_low_rank_projection} to \Cref{eq:gram_schmidt_output_projunnd}}} 
        \LineComment{$\;\;\;\;$ \textsc{projUNN-T}: $\mK = \frac{1}{2} \left[ \sum_{i=1}^k \mQ^\dagger \mU^\dagger \va_i  \vb_i^\dagger \mQ -  \mQ^\dagger \vb_i  \va_i^\dagger \mU \mQ \right]$ }
        \LineComment{$\;\;\;\;$$\;\;\;\;$ \textit{see \Cref{eq:low_rank_algebra_update}  to \Cref{eq:gram_schmidt_output_projunnt}}} 
        \EndIndent
    \State Find eigenvalues $s_1, \dots, s_{2k}$ and eigenvectors $\vu_1, \dots, \vu_{2k}$ of $\mK$
    \State Perform update step by applying eigenvalue function:
        \Indent 
        \LineComment{$\;\;\;\;$ \textsc{projUNN-D}: $\mU \leftarrow (\mU + \sum_{i=1}^k \va_i \vb_i^\dagger) \left[\mI + \sum_{j=1}^{2k} \left((s_j + 1 + \epsilon )^{-\frac{1}{2}} - 1 \right) \vu_j \vu_j^\dagger \right] $}
        \LineComment{$\;\;\;\;$$\;\;\;\;$ \textit{see \Cref{eq:eigenvalue_projUNND} and \Cref{eq:eigenvalue_projUNND_post}}, $\epsilon$ added for stability when $s_j \approx -1$ (we set $\epsilon = 10^{-8}$) }
        \LineComment{$\;\;\;\;$ \textsc{projUNN-T}: $\mU \leftarrow \mU \left[\mI + \sum_{j=1}^{2k} (\exp(-\eta s_j)-1) \vu_j \vu_j^\dagger \right]$ where $\eta$ is the learning rate}
        \LineComment{$\;\;\;\;$$\;\;\;\;$ \textit{see \Cref{eq:eigenvalue_projunnt} and \Cref{eq:eigenvalue_projunnt_post}}}        
        \EndIndent
    \EndIndent

\end{algorithmic}
\label{alg:projunn}
\end{algorithm}

\subsection{Pseudocode for performing projUNN updates}
\label{app:pseudocode}

Pseudocode for performing an update step on a unitary or orthogonal matrix $\mU$  with a gradient update of $\Delta \mU$ is shown in \Cref{alg:projunn}. In convolutional settings, the steps in \Cref{alg:projunn} are applied across blocks of the convolution in Fourier space which can be performed in parallel. As a cautionary note, especially in the last step of \Cref{alg:projunn}, where there is a composition of multiple matrix-vector multiplications, the order of these multiplications must be chosen to only perform matrix-vector operations to ensure optimal runtime. In other words, two $N \times N$ matrices should never be multiplied by each other at any point in this algorithm.

\subsection{Runtime comparisons}

\begin{figure}[t]
 \captionsetup[subfigure]{aboveskip=1pt,belowskip=-2pt}  
 \begin{subfigure}{0.5\textwidth}
    \centering
    \includegraphics[width=\linewidth]{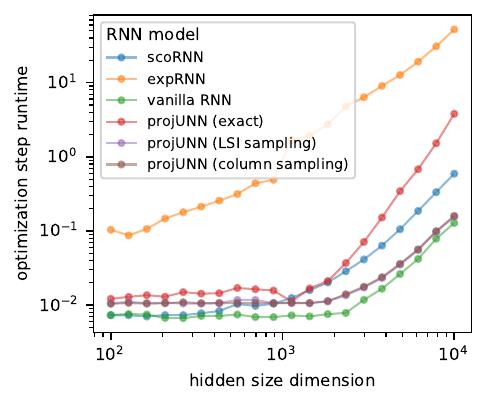}
    \caption{Runtime comparison}
    \label{fig:runtimes}
  \end{subfigure}%
  \hspace*{\fill}   
 \begin{subfigure}{0.5\textwidth}
    \centering
    \includegraphics[width=\linewidth]{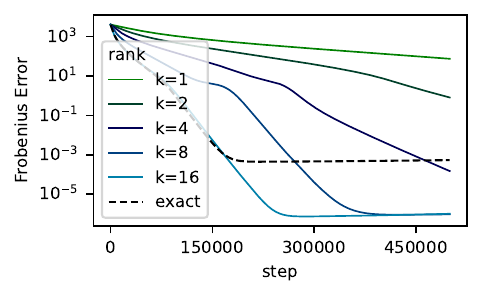}
    \caption{Learning random unitary}
    \label{fig:random_unitary_step}
  \end{subfigure}%
  \hspace*{\fill}   
  \caption{(a) Runtime of \textsc{projUNN} (with low rank approximation) scales asymptotically at same rate of a vanilla RNN and much faster than other unitary RNN models or the exact version of \textsc{projUNN} (not using low rank approximation). Practical runtime improvements are achieved when the hidden dimension is large (see \cref{app:runtime} for details). (b) \textsc{projUNN-T} can learn a random target unitary matrix using SGD. For a fixed learning rate, the loss decays at a rate proportional to the approximation rank $k$ up to $k=16$ where the approximation captures the full batch size (see exact \textsc{projUNN} which employs no approximation). The y-axis plots Frobenius error $\|\mU - \mU_{tar}\|_F^2$. }
\end{figure}

\textsc{projUNN} has a nearly optimal asymptotic runtime scaling which offers practical benefits in high dimensions. In the RNN setting, \cref{fig:runtimes} shows that the low rank version of \textsc{projUNN} has a runtime that scales at the same rate as that of a vanilla RNN albeit with increased overhead. Updating the unitary matrix of \textsc{projUNN} takes $O(kn^2)$ time for performing updates of rank $k \ll n$, only a factor $k$ more than a vanilla RNN which performs updates in $O(n^2)$ time. Note, that exact (full rank) updates to the $n \times n$ unitary matrices of a \textsc{projUNN} take roughly $O(n^3)$ time corresponding to the runtime of an SVD and equivalent to the runtime of expRNN and scoRNN \cite{lezcano2019cheap, helfrich2018orthogonal}. 

In the convolutional setting, \textsc{projUNN} offers the most benefit when there are many channels, filters with large support (very wide), or a need for exact unitary/orthogonal operations (in contrast with an approximate method like \cite{singla2021skew}). Given an $N \times N$ input with $C$ channels, a forward and backward pass of \textsc{projUNN} runs in time $O(N^2C \log(N) + kN^2C^2)$ when performing rank $k$ updates. This is a factor of $C$ faster than the Cayley implementation \cite{trockman2021orthogonalizing} which runs in time $O(N^2C^2 \log(N) + N^2C^3)$. 
For a more complete analysis of the asymptotic and empirical runtimes of various models including many not listed here, please see \cref{app:runtime}.


\section{Experiments}
\label{sec:experiments}

We propose in this section a variety of benchmarked experiments to validate the efficiency and performance of the proposed \textsc{projUNN} method focusing mostly on RNN tasks.\footnote{code repository: \small \url{https://github.com/facebookresearch/projUNN}} 
We include further details of the experiments in \Cref{app:experiment_backup} including a preliminary empirical analysis of \textsc{projUNN} in convolutional tasks. 

\paragraph{Toy model: learning random unitary}


To study the learning trajectories of \textsc{projUNN}, we consider a simple toy model aimed at learning a target random unitary. More specifically, we parameterize a large unitary matrix $\mU \in \mathbb{C}^{2048 \times 2048}$ to learn a Haar random target unitary $\mU_{tar} \in \mathbb{C}^{2048 \times 2048} $ given a dataset $\{ \vx_i,\vy_i=\mU_{tar}\vx_i\}_{i=1}^{4096}$ of size $4096$ where $\vx_i \in \mathbb{C}^{2048}$ has entries drawn i.i.d. random normal. $\mU$ is initialized as a random unitary matrix, and each step, we perform vanilla gradient descent over a batch of 16 training points using mean-squared error loss $\ell(\vx_i,\vy_i) = \|\mU\vx_i -\vy_i\|_2^2$. Approximations of rank $k$ to the gradient are obtained using the column sampling algorithm.

\cref{fig:random_unitary_step}, which plots the Frobenius error $\|\mU - \mU_{tar}\|_F^2$, shows that \textsc{projUNN-T} equipped with the column sampling approximator is able to learn the random target unitary even when $k=1$ (see \cref{app:learning_random_unitary} for plots with \textsc{projUNN-D}). Furthermore, for a fixed learning rate, learning requires fewer steps with larger $k$ up to $k=16$, the maximum rank of the gradient (note that $\nabla_{\mU}\ell(\vx_i,\vy_i)$ is rank 1). Therefore, approximating the gradient via low rank approximations can significantly speed up learning in this task (see \cref{app:learning_random_unitary} for further details).

\paragraph{Adding task}
In the adding task, an RNN must learn to add two numbers in a long sequence. We consider a variant of the adding task studied in \cite{arjovsky2016unitary}, where the input consists of two data sequences of length $T$. The first is a list of $T$ numbers sampled uniformly from $[0,1]$, and the second is a list of binary digits set to zero except for two locations (those which must be summed) set to one located uniformly at random within the intervals $[1,T/2)$ and $[T/2, T)$ respectively.

\begin{figure}[t]
    \centering
    \includegraphics{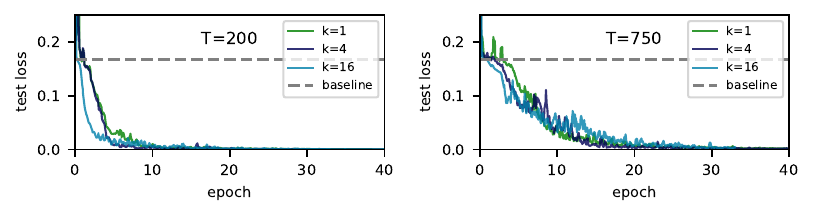}
    \caption{\textsc{projUNN-T} learns the adding task with $T=200$ and $T=750$. Test error is smoothed by taking the running average of 5 sequential points. See \cref{app:adding_task} for more details.}
    \label{fig:adding_task_main}
\end{figure}

Consistent with \cite{helfrich2018orthogonal}, we train our \textsc{projUNN-T} using an RNN with hidden dimension of 170 and the RMSprop optimizer to reduce the mean-squared error of the output with respect to the target. Naively predicting the average value of one for a random input achieves mean-squared error of approximately 0.167. As shown in \cref{fig:adding_task_main}, \textsc{projUNN-T} is able to learn the target function even with rank $k=1$ approximations. Surprisingly, for a fixed learning rate and scheduler, convergence to the true solution is almost equally fast for $k=1$, $k=4$, and $k=16$. Further details are provided in \cref{app:adding_task}.

\paragraph{Copy memory task}

\begin{figure}
    \centering
    \includegraphics{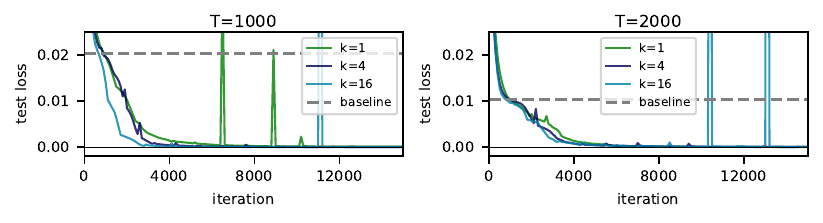}
    \caption{\textsc{projUNN-T} equipped with the column sampling approximation learns the copy task with $T=1000$ and $T=2000$ even with rank one approximations. }
    \label{fig:copy_task_main}
\end{figure}

The copying memory task is a common benchmark for RNNs \cite{hochreiter1997long, arjovsky2016unitary, henaff2016recurrent}, where the aim is to memorize input data by ignoring a long sequence of void tokens. Given an alphabet of $n+2$ symbols $\{a_i\}_{i=1}^{n+2}$, $n$ of which represent data (sequence of letters $A,B,\dots$) and additional \textit{void} (-) and \textit{start recall} (:) tokens, the RNN must output the first $K$ input tokens as the last $K$ output tokens and \textit{void} otherwise. An example input/output for $M = 6$ with $n=4$ is 
\begin{verbatim}
Input:  ABCDAD---···-------:-----
Output: ---------···-------ABCDAD
\end{verbatim} 
Here, $T=1000$ or $T=2000$ so the network must memorize data over a very long sequence of void tokens. As in \cite{jing2017tunable}, we consider $n = 8$ and input length $K=10$ and train networks with batch size $128$ using the RMSProp algorithm. Naively predicting $T+K$ void tokens followed by $K$ random selections of the $n$ possible tokens achieves a baseline loss of $K\log(n)/(T+2K)$. \textsc{projUNN-T} is able to learn the copy task efficiently as shown in \cref{fig:copy_task_main}. In fact, for fixed learning rates, rank one approximations using the column sampling algorithm provide the fastest convergence to the true solution in comparison to higher rank approximations. Networks were initialized using Henaff initialization (see \cref{app:initialization}) and the learning rate for unitary parameters was set to 32 times less than that of regular parameters (see \cref{app:copy_task} for more details).

\paragraph{Permuted MNIST}

Another challenging long-term memory task we consider is the permuted pixel-by-pixel MNIST dataset. Here, MNIST images are flattened, and pixels are randomly shuffled and placed in a sequence thereby creating some non-local dependencies. MNIST images have $28\times 28$ resolution, so the pixel-by-pixel sequences have length $T = 784$. The task is digit classification (10 classes) as in standard MNIST models. We employ the same data processing, shuffle permutation, and formatting as that in prior works \cite{lezcano2019cheap}. We perform cross-validation over different learning rates and evaluate both \textsc{projUNN-T} and \textsc{projUNN-D} with different low-rank values $k \in \{1,2,4,8,16\}$. The final test accuracy is shown in \cref{tab:mnist}. As observed in the copy and adding tasks, we find that using $k>1$ does not lead to improved performances. In fact, we provide the evolution of the test set accuracy during training in \cref{fig:curves_mnist} and note that as the number of updates is large (hundreds per epoch), even rank $k=1$ update are able to move the model's parameters to their local optimum.

\begin{table*}[t!]
\caption{Result of gradient descent optimization using the RMSprop optimizer on a single layer RNN for the permutedMNIST classification task. Each result is averaged over $3$ runs, the same cross validation is done for all settings and includes the learning rate and its schedule. Training occurs for $200$ epochs, and $10\%$ of the training set (same for all models) is set apart as validation set. The training curves are provided in \cref{fig:curves_mnist}.
    }
    \label{tab:mnist}
    \small
    \centering
    \setlength\tabcolsep{0.14em}
\begin{tabular}{|l|rrrr|rrr|rrrrr|rrrrr|}
\multicolumn{8}{c}{}&\multicolumn{5}{c}{\textsc{projUNN-D}}&\multicolumn{5}{c}{\textsc{projUNN-T}}\\
\hline
   Width &   {\small RGD} &   {\small LSTM} &   {\small ScoRNN} &   {\small ExpRNN} &   {\small DT$_{\infty}$} &   {\small DT$_{100}$} &   {\small DT$_{1}$} &   k=1 &   2 &   4 &   8 &   16 &k=1 &   2 &   4 &   8 &   16  \\
\hline
     116 &  92.5 &   91.8 &      -   &      -   &              -   &           -   &         -   &     92.8 &     93.0   &     93.0   &     92.9 &      {\bf 93.2} &                   92.5 &             92.6 &             92.5 &            93.0   &              92.8  \\
     170 &   -   &   92.0   &     94.8 &     94.9 &             95.0   &          95.1 &        {\bf 95.2} &94.3 &     94.3 &     94.4 &     94.7 &      94.3 &              94.4 &             94.3 &             94.4 &             94.1 &              94.3 \\
     360 &  93.9 &   92.9 &     96.2 &     96.2 &             {\bf 96.5} &          96.4 &        96.3 &     96.4 &     96.4 &     96.3 &     96.3 &      {\bf 96.5} &                 96.3 &             96.3 &             96.4 &             96.2 &              96.4 \\
     512 &  94.7 &   92.0   &     96.6 &     96.6 &             96.8 &          96.7 &        96.7 &     {\bf 97.0}   &     {\bf 97.0}   &     96.8 &     96.9 &      {\bf 97.0}   &                96.7 &             96.7 &             96.8 &             96.8 &              96.7  \\
\hline
\end{tabular}
\vspace{-0.3cm}
\end{table*}

\begin{figure}[t!]
    \centering
    \begin{minipage}{0.6\linewidth}
    \centering
    \includegraphics[width=\linewidth]{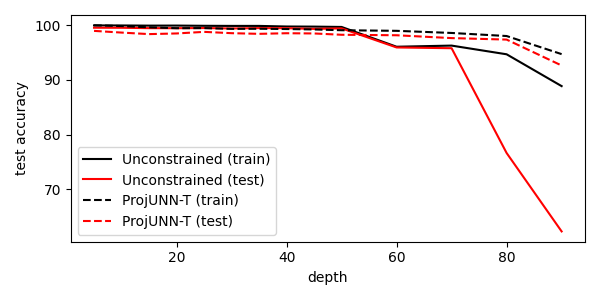}
    \end{minipage}
    \begin{minipage}{0.39\linewidth}
    \caption{\small \textsc{projUNN} can more stably train very deep CNNs. Training on MNIST is done for $50$ epochs in all cases with conv2d-BN-ReLU blocks (repeated ``depth'' times) and learning rate cross-validation (RMSprop), 32 channels throughout, and a final linear classifier. For $100$ epochs and a depth of $100$, we obtain $92.7, 23.5$ for the train/test accuracy of unconstrained CNN, and $95.7, 94.6$ for projUNN-T.}
    \label{fig:cnn_U}
    \end{minipage}
\end{figure}

\paragraph{CNN experiments} To explore the performance of our \textsc{projUNN} training algorithm for convolutional layers, we first analyzed its performance on CIFAR10 classification using a Resnet architecture \cite{he2016deep}. Our aim was not to ``beat" benchmarks but to provide an honest comparison of the performance of \textsc{projUNN} to existing methods. In fact, as noted earlier, enforcing unitarity generically results in a drop in accuracy for commonly used architectures. Consistent with prior work \cite{trockman2021orthogonalizing} we employ data-augmentation of random translations and left-right flips. Previous analysis in the RNN setting showed that rank $k=1$ is sufficient for convergence so we always set $k=1$ when using \textsc{projUNN} in the convolutional setting. For Resnet9 trained using the RMSprop optimizer, \textsc{projUNN-T} and \textsc{projUNN-D} reached $80.75\%$ and $82.06\%$ accuracy respectively, matching or outperforming reported results from existing unitary CNN models which achieved accuracies of $80.72\%$ for BCOP \cite{li2019preventing} and $81.70\%$ for Cayley \cite{trockman2021orthogonalizing} (further details in \Cref{app:cifar10_cnn}). Note, that all of these methods resulted in a performance drop compared to the standard model (without unitary constraints) which achieved accuracy of $92.26\%$. Hence, we believe that there remain a large potential for unitary models to close this gap. Separate from just performance and to motivate the use of unitary parameterization, we provide in \cref{fig:cnn_U}, test accuracy results from a simple CNN model with progressively increasing depth trained with and without unitary parameterization on MNIST data. We observe that unitary weights might provide benefits for vanilla CNN architectures that have not been designed to handle very deep settings. Of course, various techniques and tricks have been designed to enable CNNs to be trainable at large depths \cite{xiao2018dynamical,he2016deep,bjorck2018understanding}. Unitary convolutions, which are simple and theoretically motivated, can potentially be used either separately or in-tandem with these other techniques.

\vspace{-0.2cm}
\section{Discussion}
Our \textsc{projUNN} shows that one need not sacrifice performance or runtime in training unitary neural network architectures. Our results broadly take advantage of the approximate low rank structure of parameter gradients to perform updates at nearly optimal runtime. Looking beyond the setting studied here, it is an interesting question how our framework can be applied to other neural network architectures or parameter manifolds. Group convolutional neural networks and Riemannian gradient descent offer two promising avenues for further application of our techniques.

\bibliography{main.bib} 
\bibliographystyle{plain}

\section*{Checklist}


\begin{enumerate}

\item For all authors...
\begin{enumerate}
  \item Do the main claims made in the abstract and introduction accurately reflect the paper's contributions and scope?
    \answerYes{Contributions and scope of the paper are specified in the abstract and introduction.}
  \item Did you describe the limitations of your work?
    \answerYes{Limitations are addressed throughout this work, including after the descriptions of the algorithm in \Cref{sec:projunnd_alg} and \Cref{sec:projunn_conv}, in \Cref{app:numerical_stability}, in our comparison to related work in \Cref{app:prior_works}, in the runtime comparison of \Cref{app:runtime}, and various other places.}
  \item Did you discuss any potential negative societal impacts of your work?
    \answerNA{}
  \item Have you read the ethics review guidelines and ensured that your paper conforms to them?
    \answerYes{We confirm that we have read the ethics guideline to ensure our paper conforms to it.}
\end{enumerate}

\item If you are including theoretical results...
\begin{enumerate}
  \item Did you state the full set of assumptions of all theoretical results?
    \answerYes{All assumptions are given in the sections themselves or part of the statements of the relevant theorems, propositions, and lemmas.}
        \item Did you include complete proofs of all theoretical results?
    \answerYes{Proofs are included for all results and deferred to the appendix.}
\end{enumerate}

\item If you ran experiments...
\begin{enumerate}
  \item Did you include the code, data, and instructions needed to reproduce the main experimental results (either in the supplemental material or as a URL)?
    \answerYes{An anonymized demo is included in this anonymous manuscript. Full code will be made available upon publication.}
  \item Did you specify all the training details (e.g., data splits, hyperparameters, how they were chosen)?
    \answerYes{Relevant details are included in \Cref{sec:experiments} and \Cref{app:experiment_backup}. }
        \item Did you report error bars (e.g., with respect to the random seed after running experiments multiple times)?
    \answerNo{we employed the same experiment setup as prior work that consists in reporting the best test performance based on validation set performances after cross-validation. As we report such numbers for a variety of settings, error bars could be derived from our tables.}
        \item Did you include the total amount of compute and the type of resources used (e.g., type of GPUs, internal cluster, or cloud provider)?
    \answerYes{Relevant details are included in \Cref{app:network_architectures}.}
\end{enumerate}

\item If you are using existing assets (e.g., code, data, models) or curating/releasing new assets...
\begin{enumerate}
  \item If your work uses existing assets, did you cite the creators?
    \answerYes{Please refer to \Cref{app:network_architectures}.}
  \item Did you mention the license of the assets?
    \answerYes{Please refer to \Cref{app:network_architectures}.}
  \item Did you include any new assets either in the supplemental material or as a URL?
    \answerNA{}
  \item Did you discuss whether and how consent was obtained from people whose data you're using/curating?
    \answerNA{}
  \item Did you discuss whether the data you are using/curating contains personally identifiable information or offensive content?
    \answerNA{}
\end{enumerate}

\item If you used crowdsourcing or conducted research with human subjects...
\begin{enumerate}
  \item Did you include the full text of instructions given to participants and screenshots, if applicable?
    \answerNA{}
  \item Did you describe any potential participant risks, with links to Institutional Review Board (IRB) approvals, if applicable?
    \answerNA{}
  \item Did you include the estimated hourly wage paid to participants and the total amount spent on participant compensation?
    \answerNA{}
\end{enumerate}

\end{enumerate}

\clearpage
\appendix
\newgeometry{top=1in,bottom=1in,left=1in,right=1in}

\onecolumn
\numberwithin{equation}{section}

\section{Mathematical background}
\label{app:math_background}

\subsection{Lie groups and Lie algebras}
Here, we provide a brief mathematical background to Lie groups and Lie algebras with a particular focus on the unitary and orthogonal groups. For a more comprehensive overview, we recommend \cite{hall2015lie}. 

Though Lie groups are typically defined with respect to differentiable manifolds, we restrict ourselves here to the subset of matrix Lie groups which is less general but allows for a more concise and simpler theoretical overview. Informally, Lie groups are groups whose elements are specified by a set of parameters that vary smoothly and continuously, \textit{i.e.,} the group is also a differentiable manifold. Specific to matrices, Lie groups are commonly defined with respect to the general linear group $GL(n, \mathbb{C})$ denoting the set of invertible $n \times n$ matrices with complex valued entries \cite{hall2015lie}. Lie groups are closed subgroups of $GL(n, \mathbb{C})$ that have the following smoothness property.
\begin{definition}[Matrix Lie groups \cite{hall2015lie}]
    A matrix Lie group is any subgroup of $GL(n, \mathbb{C})$ with the property that any sequence of matrices $\mA_m \in \mathbb{C}^{n \times n}$ in the subgroup converge to a matrix $\mA$ that is either an element of the subgroup or not invertible (\textit{i.e.,} not in $GL(n, \mathbb{C})$).
\end{definition}

Two important Lie groups studied in this work are the unitary and orthogonal groups whose definitions are copied below.
\begin{align}
    O(n) &= \left\{ \mM \in \mathbb{R}^{n \times n} | \mM \mM^\intercal = \mM^\intercal \mM = \mI  \right\}, \\
    U(n) &= \left\{ \mM \in \mathbb{C}^{n \times n} | \mM \mM^\dagger = \mM^\dagger \mM = \mI  \right\}.
\end{align}

The Lie algebra is the tangent space of a Lie group at the identity element. To observe this, we introduce the matrix exponential map which is central to the connection between Lie groups and their corresponding Lie algebras. 
\begin{equation}
    \exp(\mM) = \sum_{k=0}^\infty \frac{\mM^k}{k!}.
\end{equation}
For compact groups, the exponential map is a smooth map whose image is the connected component to the identity of the Lie group \cite{kirillov2008introduction,hall2015lie}. The special orthogonal and unitary groups are both compact and connected so the exponential map is surjective for these groups (\textit{i.e.,} for every group element, there exists an element of the Lie algebra whose exponential is equal to that group element). However, the orthogonal group has two connected components, \textit{i.e.,} elements with positive and negative determinant, and the image of exponential map are only orthogonal matrices with positive determinant.

Since the matrix exponential map is a smooth function, we can take the derivative of the exponential map with respect to a parameter as below.
\begin{equation}
    \frac{d}{dt} \exp(t\mX) = \mX \exp(t\mX) = \exp(t\mX)  \mX,
\end{equation}
and thus,
\begin{equation}
    \frac{d}{dt} \exp(t\mX) \Bigr|_{t=0} = \mX.
\end{equation}
The above gives us the Lie algebra to a given group. 
\begin{definition}[Lie algebra \cite{hall2015lie}]
Given a Lie group $G$, the Lie algebra $\mathfrak{g}$ of $G$ is the set of matrices $\mX$ such that $e^{t\mX} \in G$ for all $t \in \mathbb{R}$.
\end{definition}
Typically, Lie algebras of a Lie group are denoted with Gothic or Fraktur font. Using the above definition, one can construct the corresponding Lie algebras. As an example, consider the unitary group where given a matrix $\mU \in U(n)$ and $\mX \in \mathfrak{u}(n)$,
\begin{equation}
    \mU^{-1} = \mU^\dagger \iff \exp(t\mX)^{-1} = \exp(-t\mX)= \exp(t\mX)^\dagger = \exp(t\mX^\dagger),
\end{equation}
and since the above holds for all $t \in \mathbb{R}$, we can differentiate the above at $t=0$, obtaining the property of elements of the unitary Lie algebra that $-\mX = \mX^\dagger$ as seen in the main text:
\begin{equation}
\begin{split}
    \frac{d}{dt} \exp(-t\mX) \Bigr|_{t=0} &= \frac{d}{dt}  \exp(t\mX^\dagger)\Bigr|_{t=0}\\
    -\mX &= \mX^\dagger.
\end{split}
\end{equation}
Proceeding in a similar fashion with the orthogonal group, we obtain the following result copied from the main text,
\begin{align}
    \mathfrak{o}(n) &= \left\{\boldsymbol{A} \in \mathbb{R}^{n \times n}: \mA + \mA^\intercal = 0 \right\}, \\
    \mathfrak{u}(n) &= \left\{\boldsymbol{A} \in \mathbb{C}^{n \times n}: \mA + \mA^\dagger = 0 \right\}.
\end{align}

\subsection{Projecting onto the group or algebra}
Projections onto the orthogonal/unitary groups or their tangent spaces are central to optimizing over the space of orthogonal/unitary matrices. We focus the discussion here to the case of unitary matrices, but note that all of the following statements apply to orthogonal matrices as well by simple adjustments such as replacing the conjugate transpose  $(^\dagger)$ with the transpose $(^\intercal)$. The tangent space $T_{\mU}$ to a matrix $\mU \in U(n)$ is equal to
\begin{equation}
    T_{\mU} U(n) = \left\{ \mX \in \mathbb{C}^{n \times n}: \mU^\dagger \mX + \mX^\dagger \mU = 0  \right\} .
\end{equation}

Given the canonical inner product $\langle \mA, \mB \rangle = \operatorname{Re}\left(\Tr[\mA^\dagger \mB]\right)$, we can show that the orthogonal projection onto the tangent space $T_{\mU} U(n)$ is equal to that given by \cref{lemma:TangentProjection} copied below.
\orthogonalProjectionTangent*
\begin{proof}
By the definition of an orthogonal projection, we need that $\Pi_{T_\mU}(\Pi_{T_\mU}(\cdot)) = \Pi_{T_\mU}(\cdot)$ and for all $\mX,\mY \in T_{\mU} U(n)$ $\Pi_{T_\mU}(\mX) = \mX$ and $\langle \Pi_{T_\mU}(\mX), \mX - \Pi_{T_\mU}(\mX) \rangle = 0$. The first two properties are straightforward to check. The last property can be shown as below using the definition of $\Pi_{T_\mU}$ and cyclic property of trace: 
\begin{equation}
    \begin{split}
        \langle \Pi_{T_\mU}(\mX), \mX - \Pi_{T_\mU}(\mX) \rangle &= \frac{1}{4}\operatorname{Re}\left(\Tr[ (\mX - \mU \mX^\dagger \mU)^\dagger (\mX + \mU \mX^\dagger \mU)  ]\right)  \\
        &= \frac{1}{4}\operatorname{Re}\left(\Tr[\mX^\dagger \mX] - \Tr[\mU^\dagger \mX \mX^\dagger \mU] + \Tr[\mX^\dagger \mU \mX^\dagger \mU] - \Tr[\mU^\dagger \mX \mU^\dagger \mX ]  \right) \\
        &= \frac{1}{4}\operatorname{Re}\left(\Tr[(\mX^\dagger \mU)^2] - \Tr[((\mX^\dagger \mU)^2)^\dagger ]  \right) \\
        &= 0.
    \end{split}
\end{equation}

Furthermore, we have for all $\mY \in T_{\mU}U(n)$ using triangle inequality and unitary invariance of the Frobenius norm:
\begin{equation}
\begin{split}
    \left\| \Pi_{T_\mU}(\mX) - \mX  \right\|_F &= \left\| \frac{1}{2}(\mX - \mU \mX^\dagger \mU) - \mX  \right\|_F \\
    &\leq \frac{1}{2} \left\| \mY - \mX \right\|_F + \frac{1}{2} \left\| \mY + \mU \mX^\dagger \mU  \right\|_F \\
    &= \frac{1}{2} \left\| \mY - \mX \right\|_F + \frac{1}{2} \left\| -\mU \mY^\dagger \mU + \mU \mX^\dagger \mU  \right\|_F \\
    &= \frac{1}{2} \left\| \mY - \mX \right\|_F + \frac{1}{2} \left\| \mY - \mX \right\|_F \\
    &= \| \mY - \mX \|_F,
\end{split}
\end{equation}
which proves that $\Pi_{T_\mU}(\cdot)$ projects onto the closest matrix $\mY \in T_{\mU}U(n)$ in Frobenius norm.
\end{proof}

Since the set of unitary/orthogonal matrices does not form a vector space, an orthogonal projection is not a well defined operation in this space. However, it is still valid to ask what is the ``closest" unitary/orthogonal matrix to a given matrix in a given norm. This is exactly what is stated in \cref{lemma:DirectProjection} copied from the main text and proven below.

\closestUnitary*
\begin{proof}
We follow the approach of \cite{keller1975closest} to prove this result. To shorten our notation, let $\mV = \mA (\mA^\dagger \mA)^{-1/2}$. Given any unitary $\mU \in U(n)$, let $\mU = \mM + \mV$ for the properly chosen $\mM \in \mathbb{C}^{n \times n}$. From unitarity of $\mU$ and $\mV$, we have 
\begin{equation}
\label{eq:relation_V_M}
    \mM\mV^\dagger + \mM\mM^\dagger + \mV\mM^\dagger = 0.
\end{equation}

Then,
\begin{equation}
\begin{split}
    \|\mA - \mU\|_F^2 &= \|\mA - \mM - \mV\|_F^2 \\
    &= \| \mA - \mV \|_F^2 + \Tr[\mM\mV^\dagger + \mM\mM^\dagger + \mV\mM^\dagger] - \Tr[ \mM^\dagger \mA + \mA^\dagger \mM] \\
    &= \| \mA - \mV \|_F^2 - \Tr[ \mM^\dagger \mA + \mA^\dagger \mM] \\
    &= \| \mA - \mV \|_F^2 - \Tr[ \mM^\dagger \mV (\mA^\dagger \mA)^{1/2}  + (\mA^\dagger \mA)^{1/2} \mV^\dagger \mM],
\end{split}
\end{equation}
and since from \cref{eq:relation_V_M} we have that $\mM\mV^\dagger + \mV\mM^\dagger = - \mM \mM^\dagger$,
\begin{equation}
\begin{split}
    \|\mA - \mU\|_F^2
    &= \| \mA - \mV \|_F^2 + \Tr[ (\mA^\dagger \mA)^{1/2} \mM \mM^\dagger ].
\end{split}
\end{equation}
The second term above is non-negative since $\Tr[ (\mA^\dagger \mA)^{1/2} \mM \mM^\dagger ] = \Tr[ \mM^\dagger (\mA^\dagger \mA)^{1/2} \mM  ]$ and $(\mA^\dagger \mA)^{1/2}$ is positive semi-definite. Thus, for all $\mU \in U(n)$,
\begin{equation}
\begin{split}
    \|\mA - \mU\|_F^2
    &\geq \| \mA - \mV \|_F^2,
\end{split}
\end{equation}
which proves the result.
\end{proof}

\section{Review of previous unitary neural network techniques}
\label{app:prior_works}

The integration of orthogonal/unitary matrices into neural networks is broadly aimed at maintaining stability in neural networks with many layers. For vanilla recurrent neural networks, repetitive application of the hidden-to-hidden transformation matrix exponentially amplifies or decays the eigenvalues of the transformation, thus resulting in exponentially large or small gradients. Similarly, in deep network architectures where weight matrices are drawn randomly, the norms of hidden states can similarly grow or decay exponentially with added layers. Enforcing unitarity or orthogonality of neural network layer transformations offers a straightforward method to address these issues of instability since orthogonal/unitary matrices have eigenvalues of unity. 

To establish notation and provide motivation for later analysis, consider a RNN whose input is a sequence of vectors $\vx(t)$ with hidden layer $\vh(t)$ updated according to the following rule:

\begin{equation}
    \vh^{(t)}=\sigma(\mM\textbf{x}^{(t)}+\mW \vh^{(t-1)})
    \label{eq:vanilla_RNN}
\end{equation}




In the unitary or orthogonal formulation for RNNs, the matrix $\boldsymbol{W}$ in \cref{eq:vanilla_RNN} is replaced by a unitary or orthogonal matrix $\mU$. During optimization of a unitary RNN, one must enforce the unitarity or orthogonality of $\mU$ during training.

Existing methods to parameterize and enforce unitarity/orthogonality in neural network layers can be separated into three categories depending on the method of parameterization. We discuss each of these in detail below:
\begin{itemize}
    \item \textbf{Layer-wise transformations:} among the first methods employed in this line of work, these methods parameterize orthogonal/unitary matrices in a layer-wise fashion where each layer is a parameterized orthogonal/unitary matrix. Example parameterizations include Givens rotations and Householder reflections. These methods are efficient when there are not many layers included in the parameterization and typically not employed when full access to the unitary/orthogonal group is required as full parameterization of an $n \times n$ unitary/orthogonal matrix requires $O(n)$ layers.
    \item \textbf{Lie algebra parameterization:} motivated by the fact that staying on the manifold of the Lie algebra is often easier than staying on the manifold of a Lie group, these methods parameterize the Lie algebra of the unitary/orthogonal groups and later obtain the actual unitary/orthogonal matrix by implementing the matrix exponential map. This matrix exponential map is typically the most costly step in these methods as one can either perform it directly (\textit{e.g.,} using an SVD) or approximate it via Taylor series or Padé approximations which require repetitive application of matrices when applying it to an input. Though we adapt techniques from these methods in our work to efficiently perform gradient updates, we do not parameterize matrices in the Lie algebra.
    \item \textbf{Matrix entry parameterization:} as in typical neural network architectures, this method parameterizes an orthogonal/unitary matrix by directly parameterizing the entries of the matrix. This method is optimal in the ``forward" direction since performing the transformation on an input simply requires matrix multiplication (as in vanilla architectures). However, updates to the matrix will no longer maintain unitarity or orthogonality, and one must employ methods to project these updates back onto the unitary/orthogonal manifold. This method is employed in our work.
\end{itemize}

We now present each of the above methods in the order given.

\paragraph{Layer-wise transformations}

Early algorithms \cite{jing2017tunable,arjovsky2016unitary,mhammedi2017efficient} maintained unitarity by parameterizing a matrix $\mU$ as a series or layered set of k parameterized unitary transformations:
\begin{equation}
    \mU = \boldsymbol{F}_{uni}^{(1)}(\theta_1) \boldsymbol{F}_{uni}^{(2)}(\theta_2) \cdots  \boldsymbol{F}_{uni}^{(k)}(\theta_k).
\end{equation}
where $\boldsymbol{F}_{uni}^{(i)}(\theta_i)$ indicates a transformation that maps parameters $\theta_i$ into a unitary matrix. These transformations include parameterized Givens rotations or Householder reflections. As a concrete example, consider the Givens rotation parameterization which is an orthogonal matrix that performs a rotation of the $i,j$-th dimensions by an amount $\theta$:
\begin{equation}
    \mG(i,j,\theta) =
    \begin{bmatrix}   1   & \cdots &    0   & \cdots &    0   & \cdots &    0   \\
                      \vdots & \ddots & \vdots &        & \vdots &        & \vdots \\
                         0   & \cdots &    \cos \theta   & \cdots &   -\sin \theta   & \cdots &    0   \\
                      \vdots &        & \vdots & \ddots & \vdots &        & \vdots \\
                         0   & \cdots &    \sin \theta   & \cdots &    \cos \theta   & \cdots &    0   \\
                      \vdots &        & \vdots &        & \vdots & \ddots & \vdots \\
                         0   & \cdots &    0   & \cdots &    0   & \cdots &    1
       \end{bmatrix},
\end{equation}
or in other words, entries $G_{ii} = G_{jj} = \cos{\theta}$, $G_{ij} = -G_{ji} = \sin{\theta}$, and $G_{kl} = \delta_{kl}$ for all other entries.

In general, at least $k=O(n)$ layers are needed to achieve a full parameterization of the $n \times n$ unitary matrices. Training is performed by updating the parameters $\theta_i$ within each layer. However, for large matrices, due to the fact that $O(n)$ layers are required to parameterize the full space of transformations, these algorithms are only efficient when parameterizations over a subset of the space of unitary/orthogonal matrices suffices. Achieving this balance of parameterization versus performance is challenging as prior work -- especially work studying the learnability of unitary matrices in quantum computation -- has shown that loss landscapes over the unitary/orthogonal manifold may contain many bad local minima \cite{kiani2020learning,frerix2019approximating,anschuetz2021critical,de2013closer}.

\paragraph{Lie algebra parameterization}
As a reminder, the Lie algebra of the orthogonal and unitary groups are the set of skew symmetric ($\mathfrak{o}(n)$) and skew Hermitian ($\mathfrak{u}(n)$) matrices,
\begin{align}
    \mathfrak{o}(n) &= \left\{\boldsymbol{A} \in \mathbb{R}^{n \times n}: \mA + \mA^\intercal = 0 \right\}, \\
    \mathfrak{u}(n) &= \left\{\boldsymbol{A} \in \mathbb{C}^{n \times n}: \mA + \mA^\dagger = 0 \right\}.
\end{align}

Transformations from the Lie algebra to the Lie group are performed using the exponential map which is surjective onto the connected components of the identity:
\begin{equation}
    \exp (\mX) = \sum_{k=0}^\infty\frac{\mX^k}{k!} = \mI + \mX + \frac{1}{2}\mX^2 + \frac{1}{6}\mX^3 + \cdots
\end{equation}

Thus, these methods parameterize the full space of unitary or special orthogonal matrices. However, the orthogonal group has two connected components (matrices with determinant equal to one and negative one) so the exponential maps only onto the positive determinant matrices. 

Note that the Lie algebra is a vector space so the sum of two matrices in the Lie algebra is also in the algebra. Thus, gradient updates can typically be very easily performed. However, the exponential map is often expensive to compute, and much prior work has focused on implementing this transformation efficiently. \cite{helfrich2018orthogonal} and \cite{singla2021skew} employ approximations to the matrix exponential via Padé or Taylor series approximations. For example, in the Taylor series approximation, one simply truncates the Taylor series to $K$ entries:
\begin{equation}
    \exp (\mX) \approx \sum_{k=0}^{K}\frac{\mX^k}{k!}, \; \; \; \; \; \left\| \exp (\mX) - \sum_{k=0}^{K}\frac{\mX^k}{k!} \right\|_2 \leq \frac{\| \mX \|_2^k}{k!},
\end{equation}
where the norm $\| \cdot \|_2$ indicates the spectral norm (\textit{i.e.,} largest singular value). This application of the matrix exponential comes at an added cost both when calculating derivatives and when applying the matrix in the forward direction (\textit{e.g.,} when one simply desires the output of a network). Applying the approximation above to the input of a layer requires $O(K)$ applications of the matrix. The matrix $\exp(\mX)$ can be explicitly constructed to avoid this added cost; however, obtaining this matrix requires a one time cost of $O(K)$ matrix-matrix multiplication operations which can be costly for large matrices. 

The Padé approximants are more typically used in approximating the exponential map since they can guarantee that the approximated matrix is actually a unitary/orthogonal matrix (as opposed to a simple Taylor series approximation which does not provide this guarantee). A Padé approximant is an optimal rational function approximation to a given function. The Padé approximant to the exponential map takes the form $\exp(\mA) \approx r_{mn}(\mA) = p_m(\mA) q_n(\mA)^{-1}$ where $p_m(\cdot)$ and $q_n(\cdot)$ are order $m$ and $n$ polynomials respectively which have closed forms:
\begin{equation}
\label{eq:pade_approx}
    p_m(\mA) = \sum_{k=0}^m \frac{(m+n-k)!m!}{(m+n)!(m-k)!k!}\mA^k, \;\;\;\;\; q_n(\mA) = \sum_{k=0}^n \frac{(m+n-k)!n!}{(m+n)!(n-k)!k!}(-\mA)^k.
\end{equation}

The error of the approximation scales as $O(\|\mA\|^{m+n+1})$ \cite{higham2009scaling}. For unitary/orthogonal matrices, this approximation has the feature that for order $m=n$, applying the approximation to an element of the Lie algebra of the orthogonal/unitary matrices outputs a orthogonal/unitary matrix. Notably, setting $m=n=1$ obtains the Cayley transform which was used in \cite{helfrich2018orthogonal}. 

Finally, we note that approximations to the exponential function often can bias the gradients so in the RNN setting, \cite{lezcano2019cheap} actually perform an SVD to obtain the actual output of the exponential map and analytically calculate gradients.

\paragraph{Matrix entry parameterization}
Parameterizing the entries of a matrix $\mU$ directly is an obvious and simple means of constructing a unitary or orthogonal matrix. However, the set of unitary/orthogonal matrices do not form an algebra so one cannot simply update these matrices by simply adding a gradient update to a given matrix. Instead, updates to these matrices must be performed by projecting the updated matrix back onto the set of unitary/orthogonal matrices. Obviously, $\mU$ must be initialized to be unitary/orthogonal in these methods.

Given a gradient update $\mG$, the matrix $\mU+ \mG$ is typically no longer unitary. Methods of Riemannian optimization are employed to update the matrix in a unitary/orthogonal fashion. One means of updating the matrix is via the Cayley transform which computes a parametric curve in the direction, employed in \cite{li2020efficient,wisdom2016full,liu2021orthogonal,trockman2021orthogonalizing}. The Cayley transform ``transports" $\mU$ in the direction of the projection of the gradient in the tangent space $\Pi(\mG)$. \textit{i.e.,} let $\mU(0)$ be the initial unitary matrix, then one can transport the matrix in the direction of $\Pi(\mG)$ (where $\Pi(\cdot)$ is the projection onto the tangent space) by a ``length" of $\alpha$ via the Cayley transform \cite{nishimori2005learning},
\begin{equation}
\label{eq:cayley_transform}
    \mU(\alpha) = \left( \mI - \frac{\alpha}{2} \Pi(\mG) \right)^{-1} \left( \mI + \frac{\alpha}{2} \Pi(\mG) \right) \mU(0).
\end{equation}
The above update formula requires a matrix inversion step which is the most costly step for large matrices. \cite{li2020efficient,liu2021orthogonal} approximate the transformation via a fixed point interation which avoids having to invert a matrix but still requires matrix-matrix multiplication. More generally, the Cayley transform is a first order Padé approximant to the exponential map which provides a connection between the matrix entry parameterization and the Lie algebra parameterization discussed previously \cite{higham2009scaling}, \textit{i.e.,} setting $m=n=1$ in \cref{eq:pade_approx} obtains the Cayley transform.

Our algorithms described in the main text directly parameterize matrix entries and thus follow this parameterization method. In performing the exponential map, we do not resort to any approximations since the exponential is efficient to perform in low rank settings. Updates to unitary matrices are either projected directly onto the closest unitary matrix in Frobenius norm (\textsc{projUNN-D}) or transported along the geodesic in the direction of the projection of the gradient onto the tangent space (\textsc{projUNN-T}). We refer the reader to the main text for a full description of our methodology.

\subsection{Orthogonal or unitary convolution}
\label{app:convolution_unitary}

Since convolutions are linear operators, one can perform orthogonal or unitary convolutions using similar techniques as in the general case studied above. As a reminder, for 2-D convolution, given input tensor $\tX \in \mathbb{C}^{ M \times N \times C}$ where $C$ denotes the number of channels of the $M \times N$ input, linear convolution (or technically cross-correlation) with a filter $\tW \in \mathbb{C}^{M \times N \times C \times C}$ takes the form
\begin{equation}
    \left[\conv_{\tW}(\tX)\right]_{p,q,d} = \left[\tW * \tX\right]_{p,q,d} = \sum_{c = 1}^{C} \sum_{m=1}^{M} \sum_{n=1}^N \tW_{m,n,d,c} \tX_{p+m,q+n,c} ,
\end{equation}
where the indexing above is assumed to be cyclic (taken modulus the corresponding dimension). Unitary or orthogonal convolutions form a subset of filters $\tW$ which preserve the norm, \textit{i.e.,} $\|\conv_{\tW}(\tX)\| = \|\tX\| $. Equivalently, $\conv_{\tW}(\cdot)$ is orthogonal/unitary if the Jacobian of the transformation is also orthogonal/unitary. 

The first method for orthogonalizing convolutions in neural networks was proposed in \cite{sedghi2018singular}. Their algorithm showed how to calculate the singular values of a linear convolution operation and then performed a series of projections on a given convolutional filter to project it onto an operator norm ball. Their algorithm made use of the convolution theorem which shows that linear convolution is diagonalized by Fourier transformations. \cite{li2019orthogonal} also proposed a method to orthogonalize convolutions by performing an SVD of the linearization of a convolution and then bounding the singular values accordingly. This method outputs linear transformations that are close to a convolution operation but no longer necessarily maintaining the equivariance of a standard convolution. Also, their method is expensive for large matrices as it requires implenting an SVD.

More recent methods implement orthogonal convolutions via approximations to the exponential map \cite{trockman2021orthogonalizing,singla2021skew}. These methods first form convolution filters $\tW$ whose Jacobians $\tJ(\tW)$ satisfy the skew symmetry property of the orthogonal group:
\begin{equation}
    \tJ(\tW) = - \tJ(\tW)^\intercal \; \equiv \; \tW = -\operatorname{conv-transpose}(\tW),
\end{equation}
where $\operatorname{conv-transpose}$ is the equivalent transposition operation in the filter space defined as \cite{singla2021skew}
\begin{equation}
    \left[\operatorname{conv-transpose}(\tL)\right]_{m,n,c,d} = \tL^*_{M-1-m,N-1-n,d,c}
\end{equation}
for a filter $\tL \in \mathbb{R}^{M \times N \times D \times C}$.

A given filter $\tW$ can be transformed into a skew symmetric convolution filter by simply applying $\tL = \tW - \operatorname{conv-transpose}(\tW)$. As in the general case, one can apply the exponential map to the convolution filter to perform orthogonal convolution.
\begin{equation}
    \exp[\conv_{\tL}] (\tX) = \tX + \tL * \tX + \frac{1}{2} \tL *^2 \tX  + \frac{1}{6} \tL *^3 \tX + \cdots ,
\end{equation}
where $*^p$ indicates the convolution is applied $p$ times, \textit{e.g.,} $\tL *^2 \tX = \tL * \tL * \tX$. Since the above operation is computationally expensive \cite{singla2021skew} implement a $k$-th order Taylor approximation to the exponential map:
\begin{equation}
    \exp[\conv_{\tL}] (\tX) \approx \sum_{p=0}^k \frac{1}{p!} \tL *^p \tX.
\end{equation}
Given 2-D inputs of dimension $N \times N \times C$, applying convolution with filters with support over $W$ elements in each dimension has runtime $O(pC^2W^2N^2)$. The added factor $p$ in the runtime is required both upon training and evaluation of the network. Though this number is held constant, this adds a multiplicative overhead to running the algorithm both when evaluating an input and when training the algorithm. \cite{singla2021skew} use the above method to implement orthogonal convolution in their skew orthogonal convolution (SOC) algorithm. Similar techniques have been used to perform invertible (not necessarily orthogonal or unitary) convolutions in \cite{hoogeboom2020convolution}. 

There are two key considerations or drawbacks associated with using the Taylor expansion to perform unitary/orthogonal convolution as in \cite{singla2021skew}. First, approximations via the Taylor approximation necessarily include an error that can be accounted for by scaling the factor $p$. However, unlike the Padé approximants discussed earlier such as the Cayley approximant, the Taylor approximation to the exponential map does not return a unitary/orthogonal operator. Thus, the value $p$ needed to bound the error from any unitary/orthogonal operator also scales logarithmically with the desired error; however, this factor can also grow with the size of the filter, the number of channels, size of the input to a convolution operation, and especially the depth of the network. Second, Taylor approximations to a function can significantly bias the gradient. As noted in \cite{lezcano2019cheap}, their expRNN algorithm avoided using the Taylor or Padé approximation in the implementation of the unitary/orthogonal operation for this reason. As an example of how approximations can fail to represent the gradient, they provide the set of functions $f_n(x)=\sin(2 \pi n x) / n$ approximating $f=0$, where $f_n \to f$ uniformly but the derivatives do not converge to zero. Especially when constructing very deep networks, such instabilities can potentially cause issues in training.

Convolution operations over filters with a large support can be performed more efficiently in the Fourier domain as explored in various prior work \cite{mathieu2013fast,ben1999fast}. Specific to orthogonal convolution, \cite{trockman2021orthogonalizing} use the 2-D fast Fourier transform to more efficiently perform orthogonal convolution. For single channel inputs and outputs, convolution in the Fourier regime corresponds to pointwise multiplication over Fourier bases. Given multi-channel inputs and outputs, convolution in the Fourier regime corresponds to matrix multiplication over the blocks indexed by channels where each block corresponds to a specific Fourier basis, \textit{i.e.,} for a filter $\tW$ and input $\tX$:
\begin{equation}
\label{eq:app_conv_fourier}
    \left[\FFT \conv_{\tW}(\tX) \right]_{\widehat{r},\widehat{s},:} = \widehat{\tW}^*_{\widehat{r},\widehat{s},:,:} \; \left[\FFT\tX \right]_{\widehat{r},\widehat{s},:} \;,
\end{equation}
where $\widehat{\tW}_{\widehat{r},\widehat{s},:,:}$ is the representation of the filter in the Fourier regime. If the filter and and input are flattened in their spatial dimensions, we can represent the operation above as block-diagonal multiplication. For example, one can visually input the Fourier multiplication on an 2-channel input with spatial dimension $[2,2]$.
\begin{align*}
\centering
\small
    \left[\FFT \conv_{\tW}(\tX) \right] = \left[\begin{array}[text width = 10mm]{cccccccc}
    \aaa \widehat{\tW}^*_{1,1,1,1} &\aaa \widehat{\tW}^*_{1,1,1,2}&&&&&&\\
    \aaa \widehat{\tW}^*_{1,1,2,1} &\aaa \widehat{\tW}^*_{1,1,2,2}&&&&&&\\
    &&\bbb \widehat{\tW}^*_{1,2,1,1} &\bbb \widehat{\tW}^*_{1,2,1,2}&&&&\\
    &&\bbb \widehat{\tW}^*_{1,2,2,1} &\bbb \widehat{\tW}^*_{1,2,2,2}&&&&\\
    &&&&\ccc \widehat{\tW}^*_{2,1,1,1} &\ccc \widehat{\tW}^*_{2,1,1,2}&&\\
    &&&&\ccc \widehat{\tW}^*_{2,1,2,1} &\ccc \widehat{\tW}^*_{2,1,2,2}&&\\
    &&&&&&\ddd \widehat{\tW}^*_{2,2,1,1} &\ddd \widehat{\tW}^*_{2,2,1,2}\\
    &&&&&&\ddd \widehat{\tW}^*_{2,2,2,1} &\ddd \widehat{\tW}^*_{2,2,2,2}\\
  \end{array}\right] 
  \left[ \begin{array}{cc}
       \aaa\widehat{\tX}_{1,1,1}  \\
       \aaa\widehat{\tX}_{1,1,2}  \\
       \bbb\widehat{\tX}_{1,2,1}  \\
       \bbb\widehat{\tX}_{1,2,2}  \\
       \ccc\widehat{\tX}_{2,1,1}  \\
       \ccc\widehat{\tX}_{2,1,2}  \\
       \ddd\widehat{\tX}_{2,2,1}  \\
       \ddd\widehat{\tX}_{2,2,2}  \\
  \end{array} \right]
\end{align*}

\cite{trockman2021orthogonalizing} use the above technique, parameterizing filters in the Lie algebra of the orthogonal group. Filters are mapped into the orthogonal group using the Cayley transform (see \cref{eq:cayley_transform} for its form). Since the Cayley transform requires matrix inversion over the matrices above, their total runtime scales as $O(N^2C^2 \log(N) + N^2C^3) $ operations for convolution over $N \times N$ images with $C$ channels. Our methodology is adapted from the techniques in \cite{trockman2021orthogonalizing} and scales more efficiently as $O(N^2C \log(N) + kN^2C^2) $ time when the rank of gradient updates $k \ll N$. We instead parameterize the convolution filter directly in the Fourier regime, by parameterizing the individual blocks ($\widehat{\tW}_{\widehat{r},\widehat{s},:,:}$) of the block diagonal matrix above. One drawback of our method compared to \cite{trockman2021orthogonalizing} is that we cannot specify the support of the convolution in the real space of the Lie algebra as we parameterize matrices directly in the Fourier space. Of course, one can perform projections of the convolutions in Fourier space onto the bases spanned by the support of the elements in the Lie algebra; however, we opt instead to implement a regularizer which biases the filter towards elements of the desired space (\textit{e.g.,} local elements).

Finally, \cite{li2019preventing} performed orthogonal convolution by parameterizing a block of matrices and corresponding projectors in their Block Convolutional Orthogonal Parametrization
(BCOP) algorithm. Their method preceded those of \cite{trockman2021orthogonalizing} and \cite{singla2021skew}. However, their method has two key drawbacks: it only parameterizes a subset of the space of orthogonal convolutions and is slower than other methods as it adds additional parameters to connect the various components in the space of orthogonal convolution. 

Finally, methods have been proposed to iteratively or approximately maintain orthogonality. \cite{bansal2018can} propose a regularizer for convolutional layers which penalize matrices having singular values far from unity. \cite{huang2020controllable} propose using a Newton's method iteration to update linear transformations to be closer to orthogonal/unitary. Their method requires iterative matrix-matrix multiplication over the unrolled weight matrix which can become expensive for large images. 

\subsection{Other related works}
In the context of recurrent neural networks, designing RNNs to learn long sequences of data has a rich history of study. Some of the first and most celebrated algorithms include the long short-term memory networks (LSTM) \cite{hochreiter1997long} and gated reccurent unit (GRU) networks \cite{cho2014learning}. Since these works, various techniques have been used to more optimally avoid issues with learning long-sequence data. Beyond the unitary RNNs discussed earlier, some work has explored using bi-directional RNNs \cite{hannun2014first,koutnik2014clockwork}, including those that have a more biologically inspired design \cite{berglund2015bidirectional}. These networks perform well on the copy task. 

Another line of research studies the stability properties of continuous state space models which can be converted into a RNN formulation \cite{engelken2020lyapunov}. Some algorithms construct continuous state space models whose attractors are stable points in the dynamical system \cite{chang2019antisymmetricrnn,erichson2020lipschitz}. More recently, continuous state space models have been designed with hidden state transformations that are customized to memorize data by limiting the learning over time to a subset of orthogonal polynomials \cite{gu2020hippo,gu2021efficiently,voelker2019legendre}. Here, hidden states are in a sense parameterized over a set of polynomial coefficients \cite{voelker2019legendre}. These algorithms perform very well on tasks such as the copy task or Permuted MNIST but do not include unitary/orthogonal transformations in their network. In fact, more recent models \cite{gu2021efficiently,voelker2019legendre} achieve slightly higher scores on the TIMIT and permuted MNIST benchmarks compared to the unitary RNN formulations studied here.

In the convolutional setting, \cite{xiao2018dynamical} form very deep convolutional neural networks by initializing parameters to construct norm-preserving orthogonal transformations. Orthogonality is not preserved during training however. \cite{wang2020orthogonal} bias filters towards orthogonality by implementing a regularizer that penalizes the weights when norms of outputs are larger or smaller than norms of inputs. The networks used in \cite{xiao2018dynamical,wang2020orthogonal} do not necessarily preserve the orthogonality property of their convolutional transformations during training. \cite{kautsky1994matrix} study orthogonal convolutions from the basis of unitary/orthogonal wavelets showing how to represent convolution in terms of these wavelets.   

Low rank approximations have been used in prior work in deep learning to prune neural network models \cite{yu2017compressing,swaminathan2020sparse} and accelerate convolutions \cite{jaderberg2014speeding,tai2015convolutional,ioannou2015training}. More related to our work, recent research has compressed gradients efficiently using low rank compression methods \cite{vogels2019powersgd}. Although their focus was in sharing compressed information across computing units, their work lends support to the notion that the information of a gradient can be effectively and efficiently stored in low rank components. Research in learning theory has also noted connections between the stable rank of neural network parameters and generalization. \cite{arora2018stronger} prove a generalization bound by compressing the models in the hypothesis class based on the stable rank of individual layers. \cite{martin2018implicit} study the phases of learning from a random matrix theory setting showing that the stable rank of a matrix tends to decay over training. Various works have studied the implicit bias induced by optimization algorithms such as gradient descent showing that in many cases, the implicit bias is towards low rank solutions \cite{gunasekar2018implicit,davenport2016overview}. Such bias towards low rank solutions has been explicitly proven in the setting of 2-layer matrix factorization \cite{li2018algorithmic}, deep matrix factorization \cite{arora2019implicit}, linear group convolutional networks \cite{lawrence2021implicit}, and matrix recovery from Pauli measurements \cite{liu2011universal}. We note that relating low-rankness to the generalization ability of learning algorithms is a richly studied topic, and there are numerous papers that we did not mention here.

Finally, a wide range of work in quantum computation and quantum machine learning studies unitary learning algorithms in the context of quantum systems \cite{biamonte2017quantum}. In fact, since the state space of closed quantum systems is transformed by unitary operators, quantum computers offer a unique platform for performing machine learning on the unitary manifold. This is an active area of research and existing methods for performing quantum machine learning on quantum architectures include variational algorithms which parameterize a quantum circuit and update the parameters via classical optimization methods \cite{cerezo2020variational,wecker2015progress,yuan2019theory,khoshaman2018quantum,kiani2021quantum,wang2019accelerated}, quantum neural networks which design analogues to classical deep neural networks \cite{killoran2019continuous,beer2020training,schuld2014quest}, and direct implementations of classical deep learning algorithms on quantum architectures \cite{kerenidis2019quantum,castelazo2021quantum,allcock2020quantum}. We stress that these quantum algorithms are inherently different in nature than their classical counterparts and there still exist significant challenges that must be surmounted before they become practically feasible \cite{cerezo2020variational,biamonte2017quantum,mcclean2018barren}. For example, the loss landscapes of quantum algorithms often have many more poor local minima in comparison to classical counterparts \cite{kiani2020learning,anschuetz2021critical} and training of quantum architectures requires sampling of outputs which is challenging when derivatives decay with the size of a model -- a phenomenon described as ``barren plateaus" \cite{mcclean2018barren,cerezo2020cost}. Furthermore, even if algorithms can be efficiently run on a quantum computer, preparing data for use in a quantum computer and reading out information from the quantum computer are challenging tasks which are not guaranteed to be efficient \cite{aaronson2015read}.

\section{Deferred proofs}
\label{app:deferred_proofs}

\subsection{Proof of \Cref{thm:DirectProjection}} 
\label{app:proof_direct_projected}
Recall \Cref{thm:DirectProjection}:
\DirectProjection*
\begin{proof}
We proceed to prove the above statement by first analyzing the case where $k=1$ and then generalizing to higher rank $k$. Recall from \Cref{lemma:DirectProjection} that we would like to perform the following update: 
\begin{equation}
\label{eq:direct_low_rank_projection}
    \mU + \mG_k \to \argmin_{\mV \in \, \mathcal{U}} \| (\mU + \mG_k) - \mV \|_F^2 = (\mU + \mG_k) \left[(\mU + \mG_k)^\dagger (\mU + \mG_k)\right]^{-\frac{1}{2}}.
\end{equation}

Let the rank one vector components of $\mG_k=\va \vb^\dagger$ and define
\begin{equation}
    \tilde{\mM} = \mU + \mG_k  = \mU + \va \vb^\dagger.
\end{equation}

With the above, we can rewrite $( \tilde{\mM}^\dagger \tilde{\mM} )^{-\frac{1}{2}}$ in \cref{eq:direct_low_rank_projection} as:
\begin{equation}
    \begin{split}
    ( \tilde{\mM}^\dagger \tilde{\mM} )^{-\frac{1}{2}} &= [ (\mU + \va \vb^\dagger)^\dagger (\mU + \va \vb^\dagger) ]^{-\frac{1}{2}}\\ &=[ \mI + \hat{\va} \vb^\dagger + \vb \hat{\va}^\dagger + c_a \vb \vb^\dagger ]^{-\frac{1}{2}},
    \label{eq:expanded_M}
\end{split}
\end{equation}
where $\hat{\va} = \tilde{\mU}^\dagger \va$ and $c_a=\va^\dagger \va$.

\cref{eq:expanded_M} is the Identity matrix plus the update of a rank two matrix. To see this, we decompose $\hat{\va}$ and $\vb$ into orthogonal components using Gram Schmidt:

\begin{equation}
    \vv_1 = \frac{\vb}{\| \vb \|} \; \; \vv_2 = \frac{ \hat{\va} - (\vv_1^\dagger \hat{\va}) \hat{\va}}{\| \hat{\va} - (\vv_1^\dagger \hat{\va}) \hat{\va} \|}.
\end{equation}{}

In this new basis:

\begin{equation}
    \hat{\va} = a_1 \vv_1 + a_2 \vv_2 \; \; \; \vb = b_1 \vv_1,
\end{equation}
and
\begin{align}
    &\tilde{\mM}^\dagger \tilde{\mM} = \nonumber \\ 
    &\mI+\begin{bmatrix}
    \vv_1 & \vv_2
    \end{bmatrix}
    \begin{bmatrix}
    a_1 b_1^* + b_1 a_1^* + c_a b_1 b_1^*  & b_1 a_2^*\\
    a_2 b_1^* & 0
    \end{bmatrix}
    \begin{bmatrix}
    \vv_1^\dagger \\
    \vv_2^\dagger
    \end{bmatrix}.
    \label{eq:gram_schmidt_output_projunnd}
\end{align}

Performing an eigendecomposition of the above $2 \times 2$ matrix into a diagonal eigenvalue matrix $\mS$ and eigenvector matrix $\mC$, we can rewrite $\tilde{\mM}^\dagger \tilde{\mM}$ in a convenient form:
\begin{equation}
\begin{split}
    \tilde{\mM}^\dagger \tilde{\mM} &= \mI + 
    \begin{bmatrix}
    \vv_1 & \vv_2
    \end{bmatrix}
    \mC
    \begin{bmatrix}
    s_1  & 0\\
    0 & s_2
    \end{bmatrix}
    \mC^\dagger
    \begin{bmatrix}
    \vv_1^\dagger \\
    \vv_2^\dagger
    \end{bmatrix}\\ &= \mI + s_1 \vu_1 \vu_1^\dagger + s_2 \vu_2 \vu_2^\dagger.
\end{split}
\end{equation}

Taking the inverse square root of the above can be performed by manipulating singular values:

\begin{align}
    ( \tilde{\mM}^\dagger \tilde{\mM} )^{-\frac{1}{2}} = \mI &+ ((s_1+1)^{-\frac{1}{2}} - 1)\vu_1 \vu_1^\dagger \nonumber \\&+ ((s_2+1)^{-\frac{1}{2}} - 1)\vu_2 \vu_2^\dagger.
    \label{eq:eigenvalue_projUNND}
\end{align}

Finally, we multiply the above on the left by $\tilde{\mM}$:

\begin{align}
    \tilde{\mM}( \tilde{\mM}^\dagger \tilde{\mM} )^{-\frac{1}{2}} = \tilde{\mM} &+ \nonumber ((s_1+1)^{-\frac{1}{2}} - 1)\tilde{\mM} \vu_1 \vu_1^\dagger \\&+ ((s_2+1)^{-\frac{1}{2}} - 1) \tilde{\mM} \vu_2 \vu_2^\dagger.
    \label{eq:eigenvalue_projUNND_post}
\end{align}

The above requires performing an eigendecomposition of a $2 \times 2$ matrix and a series of matrix-vector multiplication, matrix additions, and vector-vector outer products -- in total scaling as $O(n^2)$ time.

\paragraph{Rank $k$ updates}

Note that the above method can be extended to low rank updates, running in $O(kn^2)$ time when the rank of the update $k \ll n$. Specifically, now our update is:

\begin{equation}
    \tilde{\mM} = \mU + \sum_{i=1}^k \va_i \vb_i^\dagger.
\end{equation}

Following the same steps would ultimately require performing an eigendecomposition of a $2k \times 2k$ matrix which takes $O(k^3)$ time. Additionally, one must perform Gram-Schmidt decomposition on a set of $k$ vectors of length $n$ which takes $O(k^2n)$ time. Finally, a series of $O(k)$ matrix-vector multiplications and vector-vector outer products is performed resulting in a total runtime of $O(k(n^2 + nk + k^2))$ time. In cases where $k \ll n$, the time to perform the Gram-Schmidt decomposition and the eigendecomposition is negligible and an overall runtime of $O(kn^2)$ time is achieved. Even in cases where updates are not low rank, one can apply efficient sampling procedures (see \cref{sec:sampling}) to find low rank approximations to the update matrix and apply the methods above while maintaining runtimes.

\end{proof}

\subsection{Proof of \Cref{thm:AlgebraProjection}} 
\label{app:proof_algebra_projected}
Recall \Cref{thm:AlgebraProjection}:
\AlgebraProjection*
\begin{proof}
We would like to efficiently perform the update below:
\begin{equation}
    \mU \to \mU \exp\left[ - \eta \mU^\dagger \Pi_{T_\mU}(\mG_k) \right].
\end{equation}

As before, we proceed to prove the above statement by first analyzing the case where $k=1$ and then generalizing to higher rank $k$.

Let the rank one vector components of $\mG_k=\va \vb^\dagger$. Then
\begin{equation}
\label{eq:low_rank_algebra_update}
    \mU^\dagger \Pi_{T_\mU}(\mG_k) = \frac{1}{2} \mU^\dagger \left(\mG_k - \mU \mG_k^\dagger \mU \right) = \frac{1}{2} \left(\mU^\dagger \va \vb^\dagger - \vb \va^\dagger \mU \right) = \frac{1}{2} \left(\hat{\va} \vb^\dagger - \vb \hat{\va}^\dagger \right),
\end{equation} 
where $\hat{\va} = \mU^\dagger \va$. The above is a rank $2$ matrix. As before, we now proceed to perform an eigendecomposition in the low rank subspace of the above matrix. Using Gram Schmidt, we have
\begin{equation}
    \vv_1 = \frac{\vb}{\| \vb \|} \; \; \vv_2 = \frac{ \hat{\va} - (\vv_1^\dagger \hat{\va}) \hat{\va}}{\| \hat{\va} - (\vv_1^\dagger \hat{\va}) \hat{\va} \|}.
\end{equation}{}

In this new basis:

\begin{equation}
    \hat{\va} = a_1 \vv_1 + a_2 \vv_2 \; \; \; \vb = b_1 \vv_1,
\end{equation}
and
\begin{align}
    \mU^\dagger \Pi_{T_\mU}(\mG_k) = \frac{1}{2} \begin{bmatrix}
    \vv_1 & \vv_2
    \end{bmatrix}
    \begin{bmatrix}
    a_1 b_1^* - b_1 a_1^*  & -b_1 a_2^*\\
    a_2 b_1^* & 0
    \end{bmatrix}
    \begin{bmatrix}
    \vv_1^\dagger \\
    \vv_2^\dagger
    \end{bmatrix}.
    \label{eq:gram_schmidt_output_projunnt}
\end{align}

Performing an eigendecomposition of the above $2 \times 2$ matrix into a diagonal eigenvalue matrix $\mS$ and eigenvector matrix $\mC$, we can rewrite $\tilde{\mM}^\dagger \tilde{\mM}$ in a convenient form:
\begin{equation}
\label{eq:eigen_projunnd}
\begin{split}
    \mU^\dagger \Pi_{T_\mU}(\mG_k) &=  
    \begin{bmatrix}
    \vv_1 & \vv_2
    \end{bmatrix}
    \mC
    \begin{bmatrix}
    s_1  & 0\\
    0 & s_2
    \end{bmatrix}
    \mC^\dagger
    \begin{bmatrix}
    \vv_1^\dagger \\
    \vv_2^\dagger
    \end{bmatrix}\\ &= s_1 \vu_1 \vu_1^\dagger + s_2 \vu_2 \vu_2^\dagger.
\end{split}
\end{equation}
We apply the exponential map scaled by $\eta$ to obtain
\begin{align}
    \exp\left[ - \eta \mU^\dagger \Pi_{T_\mU}(\mG_k) \right] = \mI &+ \left(\exp(- \eta s_1) - 1\right)\vu_1 \vu_1^\dagger \nonumber \\&+ \left(\exp(- \eta s_2) - 1\right)\vu_2 \vu_2^\dagger.
    \label{eq:eigenvalue_projunnt}
\end{align}

Finally, we multiply the above on the left by $\mU$ to obtain the final result

\begin{align}
    \mU\exp\left[ - \eta \mU^\dagger \Pi_{T_\mU}(\mG_k) \right] = \mU + \left(\exp(- \eta s_1) - 1\right) \mU\vu_1 \vu_1^\dagger + \left(\exp(- \eta s_2) - 1\right) \mU\vu_2 \vu_2^\dagger.
    \label{eq:eigenvalue_projunnt_post}
\end{align}

The above requires performing an eigendecomposition of a $2 \times 2$ matrix and a series of matrix-vector multiplication, matrix additions, and vector-vector outer products -- in total scaling as $O(n^2)$ time.

\paragraph{Rank $k$ updates}

Note that as before, the above method can be extended to low rank updates, running in $O(kn^2)$ time when the rank of the update $k \ll n$. Specifically, now our update is:

\begin{equation}
    \mU \to \mU \exp \left[ -\eta \mU^\dagger \Pi_{T_\mU} \left(\sum_{i=1}^k \va_i \vb_i^\dagger \right) \right].
\end{equation}

Following the same steps, the projection onto the tangent space would be a rank $2k$ matrix. The steps that follow would ultimately require performing an eigendecomposition of a $2k \times 2k$ matrix which takes $O(k^3)$ time. Additionally, one must perform Gram-Schmidt decomposition on a set of $k$ vectors of length $n$ which takes $O(k^2n)$ time. Finally, a series of $O(k)$ matrix-vector multiplications and vector-vector outer products is performed resulting in a total runtime of $O(k(n^2 + nk + k^2))$ time. In cases where $k \ll n$, the time to perform the Gram-Schmidt decomposition and the eigendecomposition is negligible and an overall runtime of $O(kn^2)$ time is achieved. Even in cases where updates are not low rank, one can apply efficient sampling procedures (see \cref{sec:sampling}) to find low rank approximations to the update matrix and apply the methods above while maintaining runtimes.

\end{proof}

\subsection{First order equivalence to \textsc{projUNN-D}} 
\label{app:first_order_equivalence}
Though \textsc{projUNN-D} and \textsc{projUNN-T} perform different updates, we can show that up to first order, the updates are in fact equivalent. Furthermore, as one may expect, this first order update is equal to the projection of the gradient onto the tangent space (see \Cref{lemma:TangentProjection}). In the case of \textsc{projUNN-D}, this shows that the update step is a retraction or first order approximation to the matrix exponential implemented in \textsc{projUNN-T} \cite{bonnabel2013stochastic}. 
\begin{restatable}[First order equivalence]{proposition}{FirstOrderUpdate}
\label{prop:first_order_update}
For an $n \times n$ unitary/orthogonal matrix $\mU$ perturbed by $\Delta\mU$,  gradient updates applied by algorithms \textsc{projUNN-D} (\cref{eq:direct_update}) and \textsc{projUNN-T} (\cref{eq:algebra_update_rule}) are, up to first order, equal to $\mU + \Pi_{T_\mU}(\Delta\mU)$, \textit{i.e.,}
\begin{equation*}
    \mU  \to \mU + \frac{1}{2}\left( \Delta\mU - \mU \Delta  \mU^\dagger \mU \right) + O(\Delta\mU \Delta\mU^\dagger).
\end{equation*}
\end{restatable}
\begin{proof}
We first show that the above formula holds for \textsc{projUNN-D} and then show the same for \textsc{projUNN-T}. Recall the update formula for \textsc{projUNN-D}:

\begin{equation}
    \mU + \Delta\mU \to \argmin_{\mV \in \, \mathcal{U}} \| (\mU + \Delta\mU) - \mV \|_F^2 = (\mU + \Delta\mU) \left[(\mU + \Delta\mU)^\dagger (\mU + \Delta\mU)\right]^{-\frac{1}{2}}.
\end{equation}
Expanding the above up to first order and applying the first order Taylor expansion of $(\cdot)^{-1/2}$, we have:
\begin{equation}
\begin{split}
    (\mU + \Delta\mU) \left[(\mU + \Delta\mU)^\dagger (\mU + \Delta\mU)\right]^{-\frac{1}{2}} &= (\mU + \Delta\mU) \left[\mI + \mU^\dagger \Delta\mU + \Delta\mU^\dagger \mU + O(\Delta\mU^\dagger \Delta\mU) \right]^{-\frac{1}{2}} \\
    &= (\mU + \Delta\mU) \left[\mI - \frac{1}{2}\left( \mU^\dagger \Delta\mU + \Delta\mU^\dagger \mU \right) + O(\Delta\mU^\dagger \Delta\mU) \right] \\
    &= \mU + \frac{1}{2}\left( \Delta\mU - \mU \Delta\mU^\dagger \mU \right) + O(\Delta\mU^\dagger \Delta\mU).
\end{split}
\end{equation}
    
Similarly, for \textsc{projUNN-T}, recall the update formula (ignoring $-\eta$ term for learning rate):
\begin{equation}
    \mU \to \mU \exp\left[ \mU^\dagger \Pi_{T_\mU}(\Delta\mU) \right].
\end{equation}

Expanding the above and applying the first order approximation of the matrix exponential,
\begin{equation}
    \begin{split}
        \mU \exp\left[ \mU^\dagger \Pi_{T_\mU}(\Delta\mU) \right] &= \mU \exp\left[ - \eta \mU^\dagger \frac{1}{2}\left( \Delta\mU - \mU \Delta\mU^\dagger \mU \right)  \right] \\
        &= \mU\left[ \mI + \frac{1}{2}\left( \Delta\mU - \mU \Delta\mU^\dagger \mU \right) + O(\Delta\mU^\dagger \Delta\mU) \right] \\
        &=\mU + \frac{1}{2}\left( \Delta\mU - \mU \Delta\mU^\dagger \mU \right) + O(\Delta\mU^\dagger \Delta\mU).
    \end{split}
\end{equation}

\end{proof}

\subsection{Unitary convolutional manifold is connected}
\label{app:unitary_connected}

Before proceeding to prove that the space of unitary convolutions is connected, we first provide a version of the convolution theorem which shows that convolution in the Fourier domain corresponds to block multiplication over channels. This is a classic result also contained in various prior works \cite{sedghi2018singular,trockman2021orthogonalizing}, and we provide a short proof here for completeness.

\begin{lemma}
\label{lem:conv_theorem}
Given convolution filter $\tW \in \mathbb{C}^{M \times N \times C \times C}$ and input $\tX \in \mathbb{C}^{M \times N \times C}$, recall the definition of cyclic convolution (or technically cross-correlation) of $\tW$ and $\tX$:
\begin{equation}
   \left[\conv_{\tW}(\tX) \right]_{p,q,d} = \sum_{c = 1}^{C} \sum_{m=1}^{M} \sum_{n=1}^N \tW_{m,n,d,c} \tX_{p+m,q+n,c}.
\end{equation}
Let $\FFT$ be the two dimensional fast Fourier transform, then convolution in the Fourier domain corresponds to block-wise multiplication over channels:
\begin{equation}
    \left[\FFT \conv_{\tW}(\tX) \right]_{\widehat{r},\widehat{s},:} = \widehat{\tW}^*_{\widehat{r},\widehat{s},:,:} \; \left[\FFT\tX \right]_{\widehat{r},\widehat{s},:}.
\end{equation}
\end{lemma}
\begin{proof}
Let the roots of unity be denoted as $\omega_M=e^{2\pi i/M}$ and $\omega_N=e^{2\pi i / N}$. Then,
\begin{equation}
    \begin{split}
        \left[\FFT \conv_{\tW}(\tX) \right]_{\widehat{r},\widehat{s},d} &= \sum_{u=1}^M \sum_{v=1}^N \omega_M^{u\widehat{r}} \omega_M^{v\widehat{s}} \sum_{m =1}^M \sum_{n=1}^N \sum_{c=1}^C \tW_{m,n,d,c} \tX_{u+m,v+n,c} \\
        &= \sum_{u=1}^M \sum_{v=1}^N \sum_{m =1}^M \sum_{n=1}^N \sum_{c=1}^C \omega_M^{u\widehat{r}} \omega_M^{v\widehat{s}} \omega_M^{\widehat{r}m} \omega_M^{-\widehat{r}m} \omega_N^{\widehat{s}n} \omega_N^{-\widehat{s}n} \tW_{m,n,d,c} \tX_{u+m,v+n,c} \\
        &= \sum_{c=1}^C \sum_{m =1}^M \sum_{n=1}^N \omega_M^{-\widehat{r}m} \omega_N^{-\widehat{s}n} \tW_{m,n,d,c} \sum_{u=1}^M \sum_{v=1}^N   \omega_M^{\widehat{r}(u+m)} \omega_M^{\widehat{s}(v+n)} \tX_{u+m,v+n,c}  \tX_{u+m,v+n,c} \\
        &=\sum_{c=1}^C \widehat{\tW}_{\widehat{r},\widehat{s},d,c}^* \widehat{\tX}_{\widehat{r},\widehat{s},c}.
    \end{split}
\end{equation}

\end{proof}
As an aside, the complex conjugation of the filter above is due to the fact that convolution in neural networks corresponds to the more commonly used term cross-correlation in mathematics. If the convolution operation were redefined as what is more commonly known as convolution in mathematics, \textit{i.e.,} define $\conv'$ as $\left[\conv'_{\tW}(\tX) \right]_{p,q,d} = \sum_{c = 1}^{C} \sum_{m=1}^{M} \sum_{n=1}^N \tW_{m,n,d,c} \tX_{p-m,q-n,c}$, then the complex conjugate would no longer appear on the filter term.

From here, we simply show that each block in the above is connected which allows us to prove that the unitary manifold is connected.
\ConnectedConvolution*
\begin{proof}
Since the fourier transform is unitary and invertible, we can represent every filter $\tW$ in the fourier domain as $\widehat{\tW} \in \mathbb{C}^{M \times N \times C \times C}$ and vice-versa. In the Fourier domain, given \cref{lem:conv_theorem}, we have that convolution corresponds to block-wise multiplication over channels. For unitary convolution, each block $\widehat{W}_{\widehat{r},\widehat{s},:,:}$ indexed by frequencies $\widehat{r}$ and $\widehat{s}$ must be a unitary matrix, so we now analyze the set of filters whose blocks are unitary matrices in the Fourier domain.

The space of unitary matrices $U(C)$ is connected \cite{hall2015lie}. Therefore for every block, there exists a connected path between any two unitary matrices $\widehat{W}_{\widehat{r},\widehat{s},:,:}^{(1)},\widehat{W}_{\widehat{r},\widehat{s},:,:}^{(2)}\in U(C)$. The full space of unitary convolutions is parameterized by the $MN$ times direct product of groups $U(C)$, \textit{i.e.,} $U(C)^{(\times MN)}$. Since the direct product of finitely many connected spaces is also connected, then the space of unitary convolutions is connected. 
\end{proof}

\section{Analysis on various benchmarked tasks}
\label{app:experiment_backup}

\subsection{Learning random unitary}
\label{app:learning_random_unitary}
\begin{figure}
    \centering
    \includegraphics[]{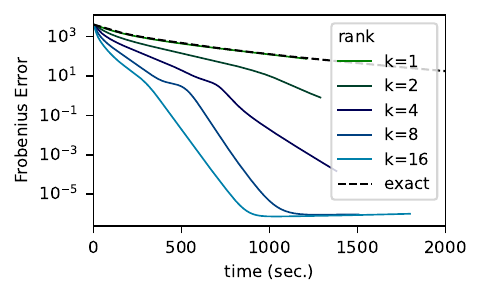}
    \caption{ Runtime of \textsc{projUNN-T} in learning a random target unitary matrix is faster when using low rank approximations. Here we plot Frobenius error $\|\mU - \mU_{tar}\|_F^2$ over the course of optimization. The learning rate is fixed for each value of $k$.
    }
    \label{fig:random_unitary_runtime}
\end{figure}

In addition to the analysis shown in the main text, here we include a plot showing the value of the Frobenius error as a function of the runtime (\cref{fig:random_unitary_runtime}). Using \textsc{projUNN-T} to peform learning via low rank approximations to the gradient significantly speed up learning. Optimization was performed with vanilla gradient descent over batch sizes of 16 with learning rates of 0.5 and 0.33 for \textsc{projUNN-T} and \textsc{projUNN-D} respectively. Since the learning rate was fixed across all $k$ for each of these experiments, the norm of the update was smaller for lower values of $k$. Scaling up the learning rate based on the value of $k$ could make runtimes even faster for lower values of $k$.

\begin{figure}
    \centering
    \includegraphics[]{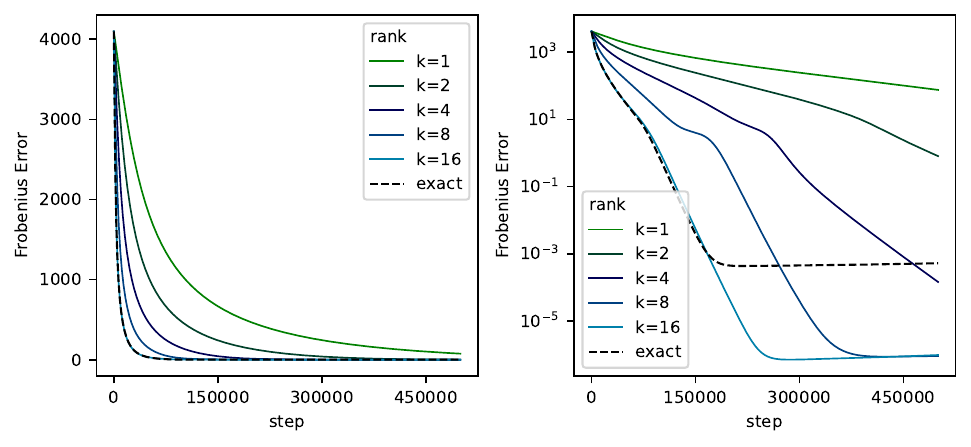}
    \caption{ Learning trajectory of \textsc{projUNN-T} equipped with the column sampling approximation in the random unitary learning task. 
    }
    \label{fig:random_unitary_T_column}
\end{figure}

\begin{figure}
    \centering
    \includegraphics[]{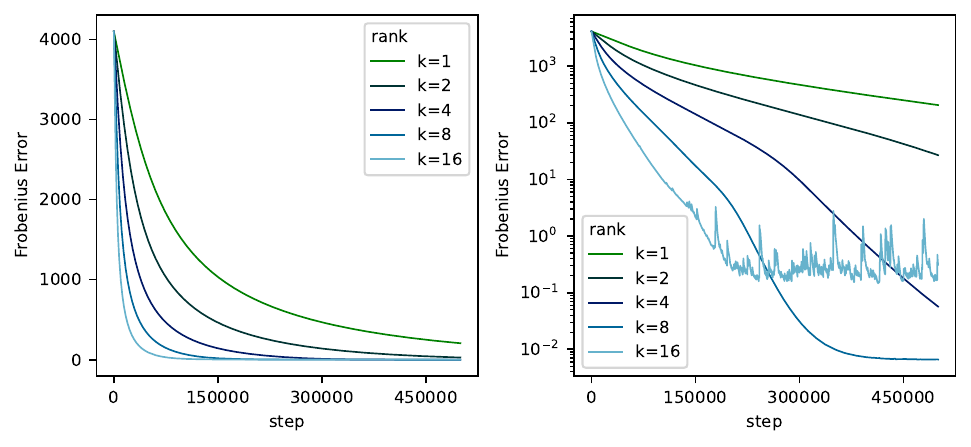}
    \caption{ Learning trajectory of \textsc{projUNN-D} equipped with the column sampling approximation in the random unitary learning task. Performing gradient updates using exact formulas was too computationally expensive for \textsc{projUNN-D} and thus not included here.
    }
    \label{fig:random_unitary_D_column}
\end{figure}

\begin{figure}
    \centering
    \includegraphics[]{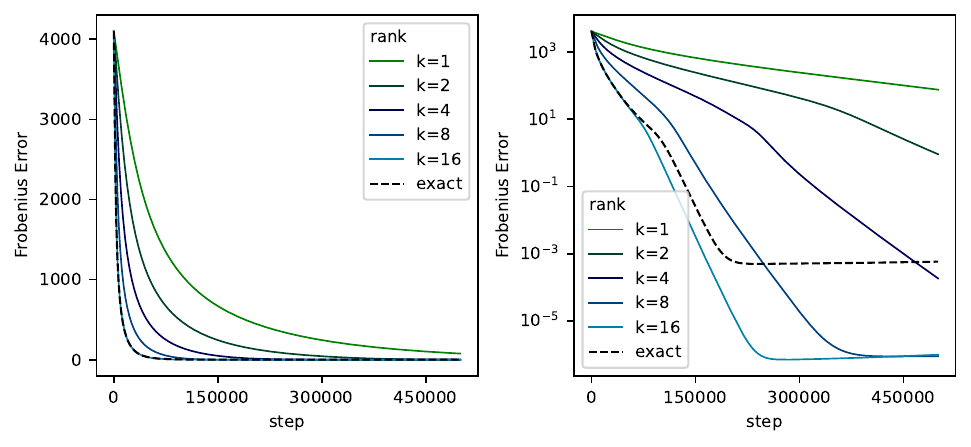}
    \caption{ Learning trajectory of \textsc{projUNN-T} equipped with the LSI sampling approximation in the random unitary learning task. 
    }
    \label{fig:random_unitary_T_LSI}
\end{figure}

\begin{figure}
    \centering
    \includegraphics[]{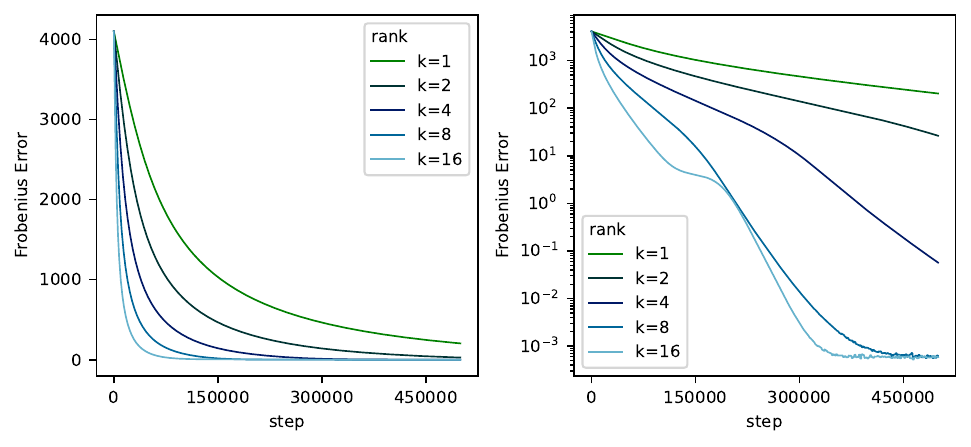}
    \caption{ Learning trajectory of \textsc{projUNN-D} equipped with the LSI sampling approximation in the random unitary learning task. Performing gradient updates using exact formulas was too computationally expensive for \textsc{projUNN-D} and thus not included here.
    }
    \label{fig:random_unitary_D_LSI}
\end{figure}

For sake of completeness, we also include plots (in both regular and logarithmic scaling axis) showing the learning trajectory of the various combinations of samplers and \textsc{projUNN} architectures in the random unitary task. See \cref{fig:random_unitary_T_column} for \textsc{projUNN-T} with column sampling (left hand side repeated from main text), \cref{fig:random_unitary_D_column} for \textsc{projUNN-D} with column sampling, \cref{fig:random_unitary_T_LSI} for \textsc{projUNN-T} with LSI sampling, and \cref{fig:random_unitary_D_LSI} for \textsc{projUNN-D} with LSI sampling.

\cref{fig:random_unitary_D_column} and \cref{fig:random_unitary_D_LSI} also highlight potential instabilities in training \textsc{projUNN-D} over long periods of time, a feature noted in the main text and discussed further in \cref{app:numerical_stability}. To alleviate this, for \textsc{projUNN-D} architectures, we set the learning rate to slightly lower at 0.33 and projected parameter matrices onto the closest unitary using \cref{lemma:DirectProjection} every 2048 steps. Note, that since this projection step was performed sparingly every $n$ steps (where $n$ is the dimension of the matrix), we did not find any significant decrease in runtime.

\subsection{Adding task}
\label{app:adding_task}

\begin{figure}
    \centering
    \includegraphics{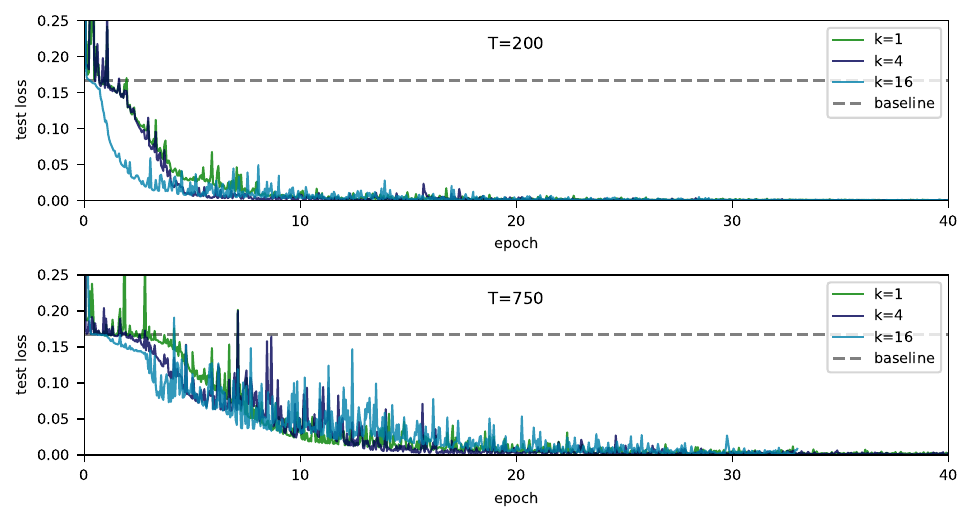}
    \caption{ \textsc{projUNN-T} (with column sampling approximation) is effective at learning the adding task. Above is a copy of \cref{fig:adding_task_main} except the plot here is expanded in size and test error is not smoothed. }
    \label{fig:adding_task_expanded_appendix}
\end{figure}

\begin{figure}
    \centering
    \includegraphics{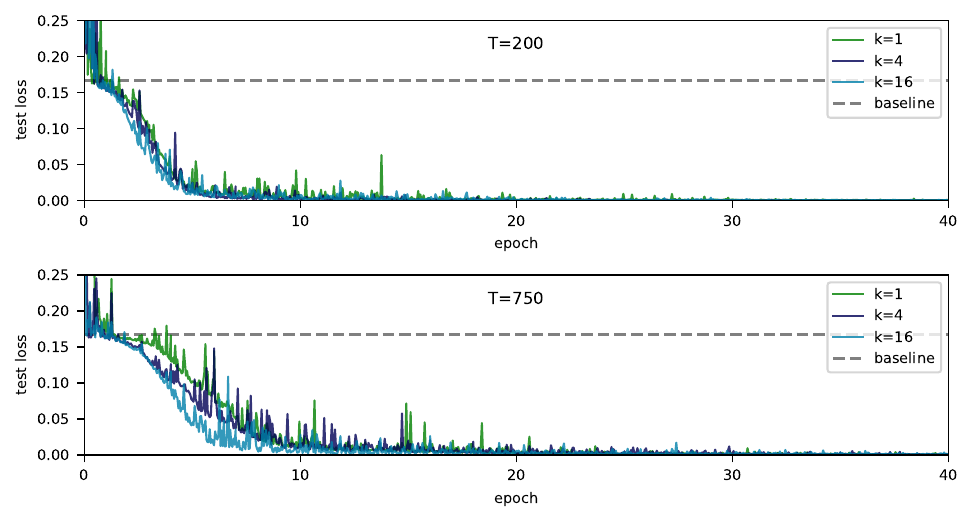}
    \caption{ \textsc{projUNN-T} (with LSI sampling approximation) is effective at learning the adding task. }
    \label{fig:adding_task_expanded_LSI}
\end{figure}

As shown in the main text, our \textsc{projUNN} is able to learn the adding task even for long sequence lengths. Rank $k$ gradient approximations were obtained using the column sampling approximation. RMSprop was used to train the RNN in this task. The learning rate was initialized to $0.001$ for non-orthogonal parameters and $0.001/32$ for all orthogonal parameters. Each epoch, the learning rate was reduced by multiplying it by $0.96$. We found that initializing the orthogonal matrix as the identity matrix worked well for this task. This is in distinction with \textit{e.g.,} \cite{helfrich2018orthogonal} which initialized using the Cayley initialization (block diagonal). For reference, we also include here an enhanced version of the plot in the main text in \cref{fig:adding_task_expanded_appendix} and a similar plot obtained using \textsc{projUNN} with LSI sampling in \cref{fig:adding_task_expanded_LSI}.

\subsection{Copy task}
\label{app:copy_task}
\cref{fig:copy_task_main} in the main text shows that our \textsc{projUNN-T} can efficiently learn the copy task. For this task, we set the learning rate of the RMSprop optimizer to $7e-4$. For orthogonal parameters, this learning rate was divided by 32. Orthogonal matrices were intialized using Henaff initialization \cite{henaff2016recurrent} as described in \cref{app:initialization}. We found that this initialization scheme worked best in comparison to other methods.

\subsection{Pixel permuted MNIST}
\label{app:permuted_mnist}

\cref{fig:curves_mnist} charts the trajectory of learning on the permuted MNIST task. As is evident in the figure, rank $k=1$ is sufficient to guarantee optimal or nearly optimal convergence in all settings.

\begin{figure}[t!]
    \centering
    \begin{minipage}{0.49\linewidth}
    \centering
    \textsc{projUNN-D}\\
    \includegraphics[width=\linewidth]{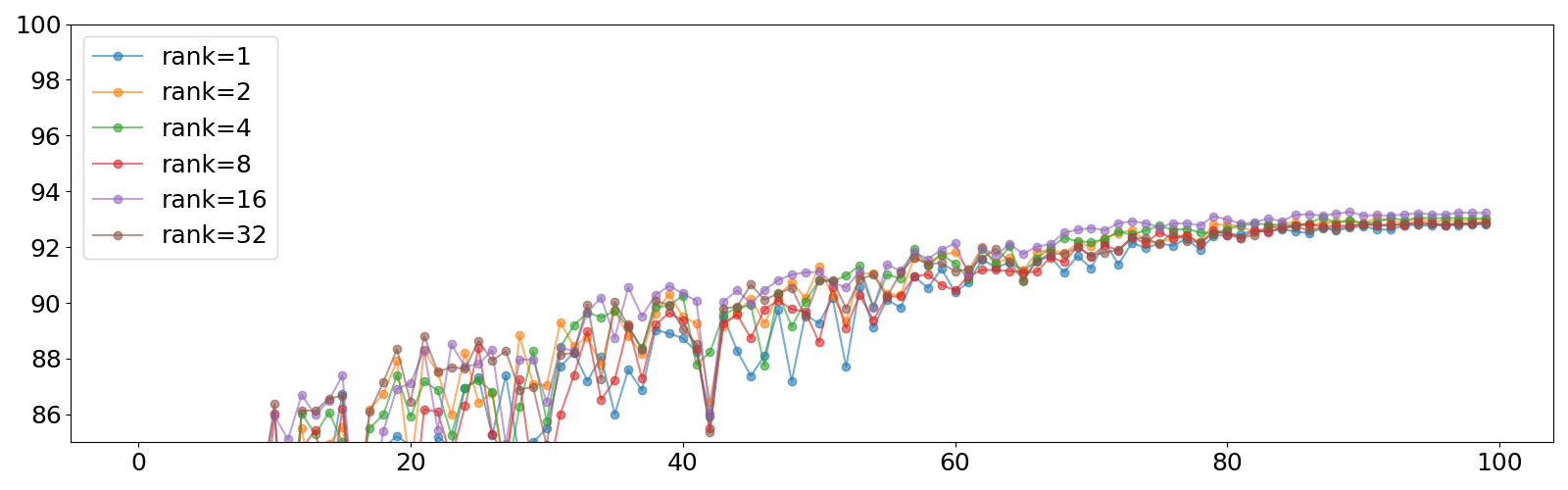}\\
    \includegraphics[width=\linewidth]{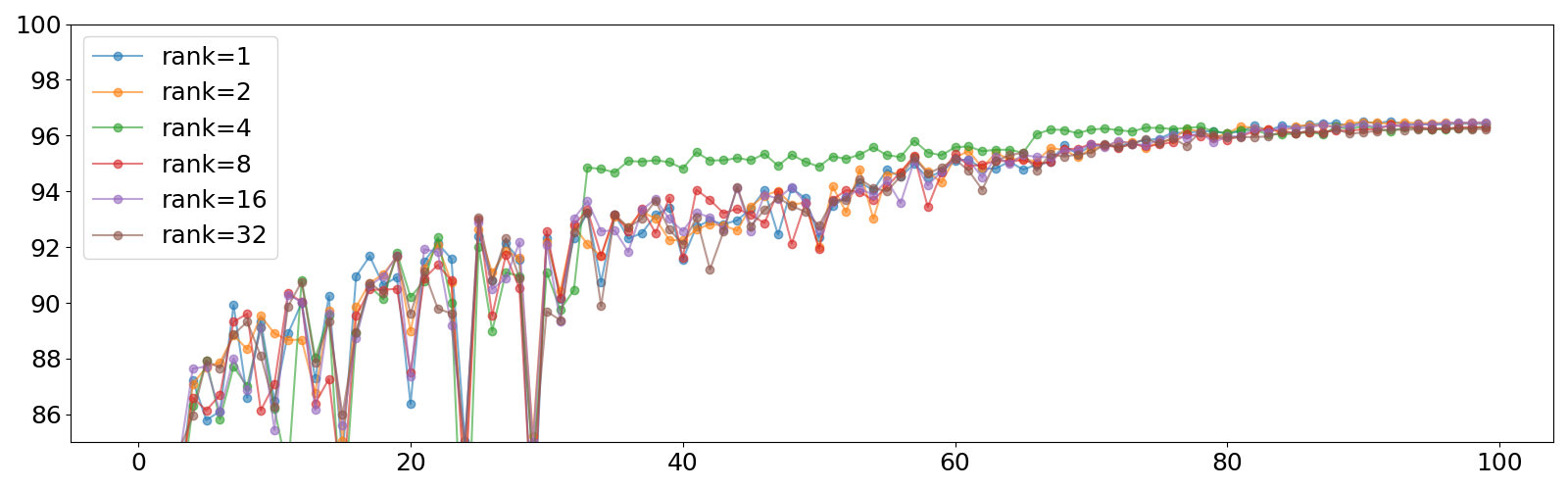}\\
    \includegraphics[width=\linewidth]{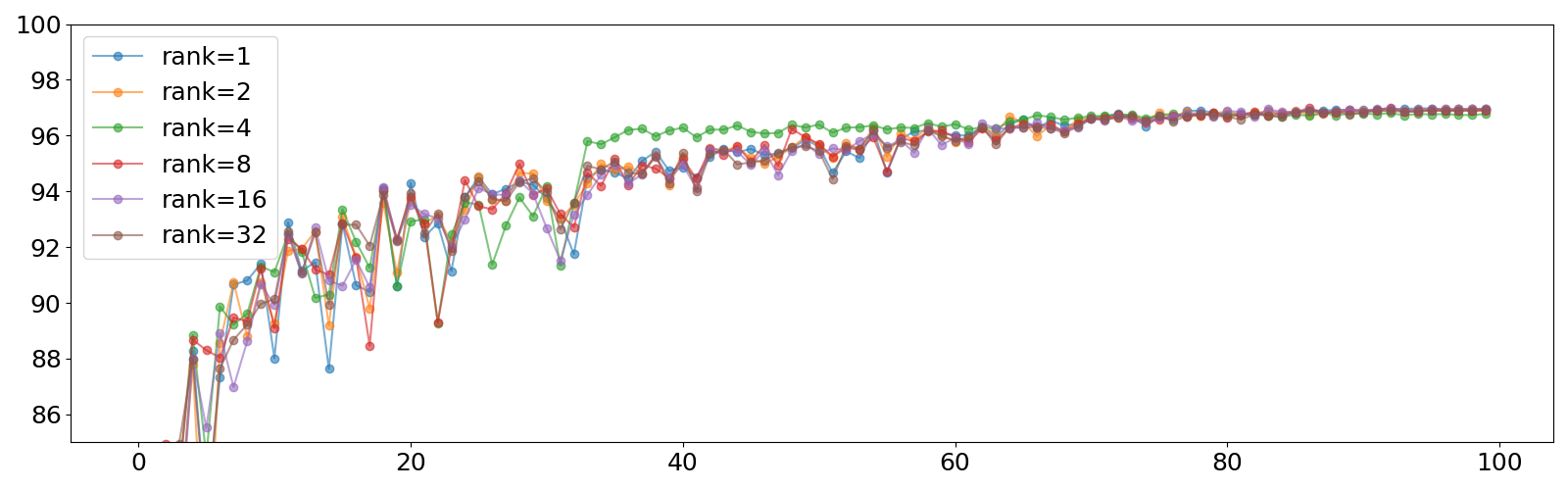}
    \end{minipage}
    \begin{minipage}{0.49\linewidth}
    \centering
    \textsc{projUNN-T}\\
    \includegraphics[width=\linewidth]{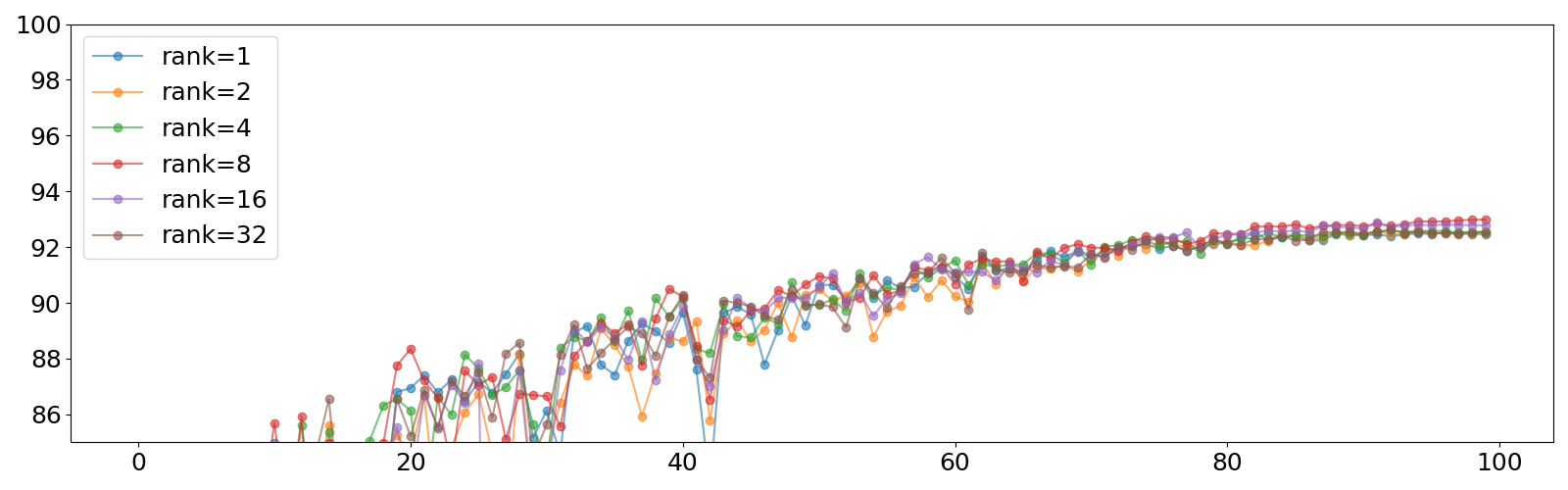}\\
    \includegraphics[width=\linewidth]{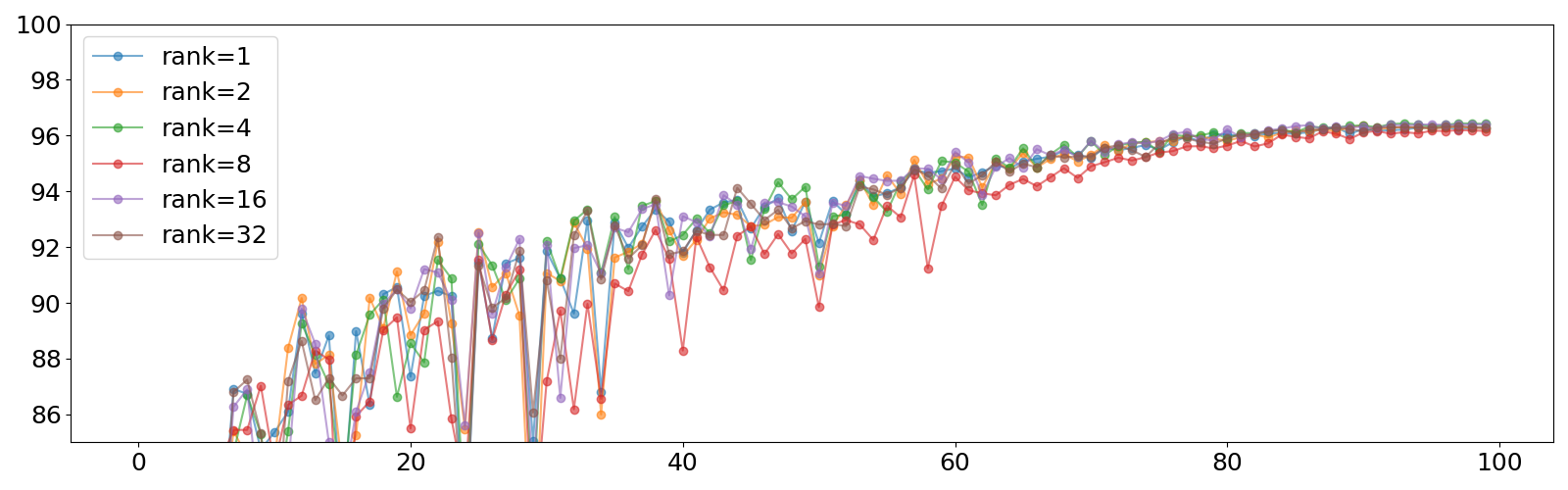}\\
    \includegraphics[width=\linewidth]{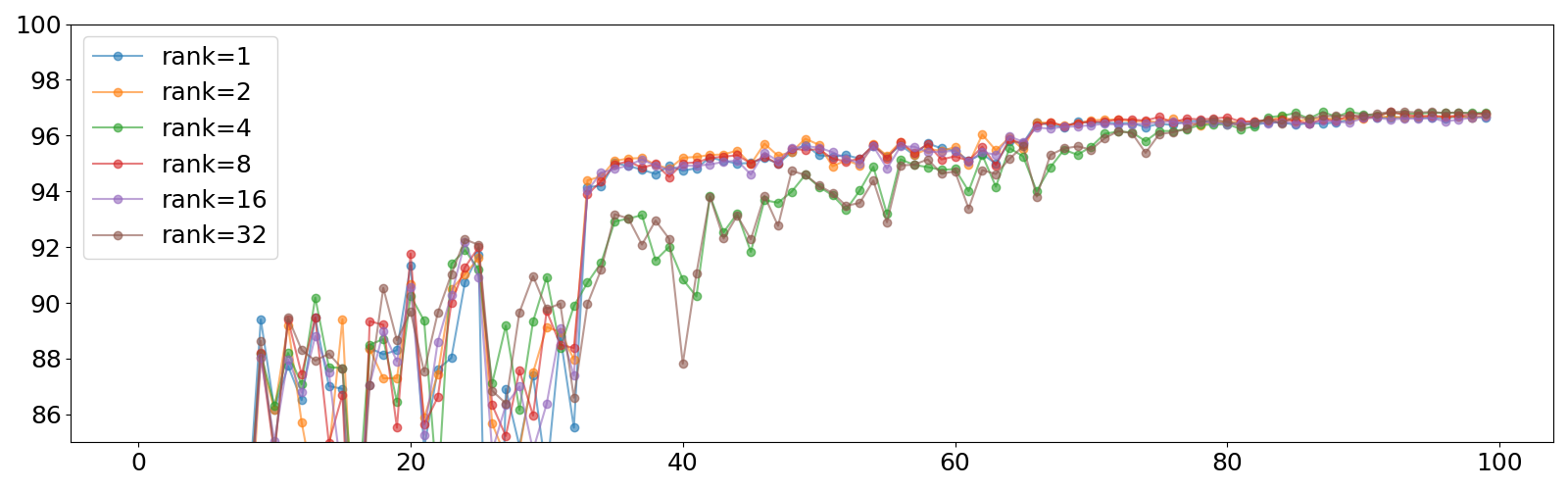}
    \end{minipage}
    \caption{Evolution of the test set accuracy during training at each epoch for the pixel-MNIST task. We depict the evolution for different RNN width (from top to bottom: 116,360 and 512). We observe that regardless of the rank $k$ of the projUNN update, we reach the same final performances.}
    \label{fig:curves_mnist}
\end{figure}

\subsection{CIFAR10 CNN experiments}
\label{app:cifar10_cnn}

\begin{table}
    \caption{Test accuracy on CIFAR10 with different unitary convolution parameterizations and our proposed \textsc{projUNN} algorithm. We stress that SOC is an approximate unitary parameterization.}
    \label{tab:conv_table}
    \centering
    \setlength\tabcolsep{0.2em}
    \begin{tabular}{c|c|c|c|c|c|c}
           \multicolumn{1}{c}{} &\multicolumn{5}{c}{Convolution Type}\\ \toprule
         Model&  Standard & BCOP & SOC &Cayley& Proj-D & Proj-T \\ \toprule
         Resnet9& 92.26&80.72&-&81.70&80.75&{\bf 82.06}\\ \midrule
         Resnet18& 95.10&92.38&{\bf 94.24}&-&89.43&89.59\\\bottomrule
    \end{tabular}
\end{table}

To explore the performance of our \textsc{projUNN} training algorithm for convolutional layers, we analyze its performance on CIFAR10 classification. Here, we provide further details to the results in the main text as well as some further preliminary experiments on unitary CNNs in resnet architectures. As in prior work \cite{trockman2021orthogonalizing} we employ the usual data-augmentation of random translations and left-right flips. We leverage the Resnet \cite{he2016deep} architecture. As our previous analysis in the RNN setting has shown that rank $k=1$ is sufficient for convergence, we always set $k=1$ when using \textsc{projUNN} in the convolutional setting. We leverage the RMSprop optimizer and perform cross-validation on the learning rate. We present our results in \cref{tab:conv_table} and make two observations. First, the smaller model (Resnet9) is able to reach or slightly outperform the alternative exact orthogonal constraints. Second, the larger model (Resnet18) falls behind the approximate orthogonal constraint method. This result is perhaps expected as our convolutional layers are full-width, \textit{i.e.,} they allow for much greater degree of over-fitting. This was not detrimental in the small model with fewer convolutional layers. As a result, although we validate the ability of \textsc{projUNN} to produce state-of-the-art small convolutional networks, there remains open avenues of research to extend the method to larger models. As an aside, the resnet architecture is, in a sense, a very stable architecture since it is precisely designed to be able to incorporate many layers. To provide an honest comparison to prior work, we analyzed the performance of \textsc{projUNN} with respect to this resnet architecture, but note that more ``vanilla" CNN architectures may be better targets for unitary constraints.

\section{Analysis of low rank updates}
\label{sec:numerical_rank}
It is typically the case that gradients of matrices with respect to a loss function are not \emph{exactly} low rank but \emph{approximately} low rank. More specifically, a matrix $\mA$ is \emph{approximately} low rank if there exists a rank $k$ matrix $\mA_k$ such that the relative error of the approximation $E_{rel}$ is small:
\begin{equation}
    E_{rel} = \frac{\| \mA - \mA_k \|_F }{\|\mA\|_F}
\end{equation}
where $\| \cdot \|_F$ denotes the Frobenius norm of a matrix. We note that there is a connection between the above and the stable rank of a matrix $\mA$ defined as  $\|\mA\|_F^2/\|\mA\|_2^2$. The stable rank is upper bounded by the exact or hard rank of a matrix $\mA$.

In RNN settings, \cref{fig:rank_approximation} which plots $1-E_{rel}$ shows that low rank approximations to a gradient can typically be very close to the true gradient even as the hidden size of the RNN is large. In \cref{fig:rank_approximation}, the RNN was given an input of batch size $64$ where each input is a sequence of length $20$ and each element of the sequence is a vector in $\mathbb{R}^{200}$ with elements drawn from the standard normal distribution. The RNN outputs logits to learn random target integers ranging from $1$ to $10$. Gradients were calculated for a single optimization step over a single batch of the random input and output data.

\begin{figure}
\hspace*{-0.5cm}
    \centering
    \begin{subfigure}[t]{0.49\textwidth}
    \centering
    \includegraphics{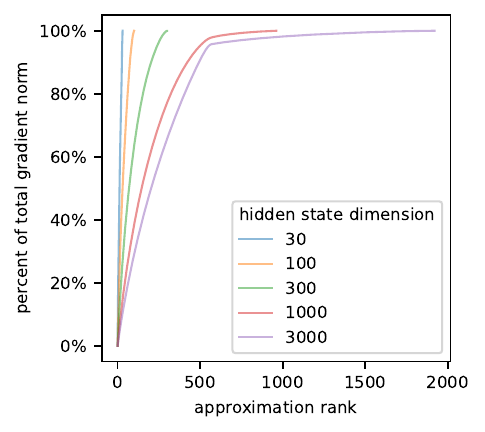}
    \caption{unnormalized}
    \label{fig:rank_approximation_unnormalized}
    \end{subfigure}
    \begin{subfigure}[t]{0.49\textwidth}
    \centering
    \includegraphics{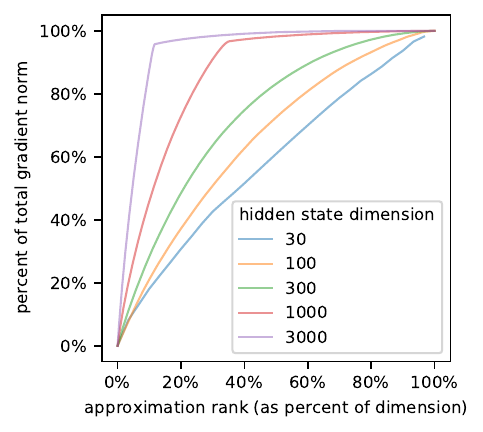}
    \caption{normalized by dimension}
    \label{fig:rank_approximation_normalized}
    \end{subfigure}
    \caption{Portion of Frobenius norm (1-$E_{rel}$) captured by a low rank approximation to the gradient of the matrix typically improves with the dimension of the matrix. For larger matrices, at least half of the Frobenius norm of the gradient can be captured with a low rank approximation of a small percent of the overall dimension. Here, we assume that low rank approximations to the matrix are optimal and the RNN is trained on random inputs and outputs. Plots are for the matrix performing transformation from hidden states to hidden states (\textit{i.e.,} the matrix replaced by a unitary/orthogonal matrix in \textsc{projUNN} implementations. 
    }
    \label{fig:rank_approximation}
\end{figure}

As observed in the main text, such low rank behavior is often more evident in real-world data. In fact, as seen in \cref{fig:cifar_low_rank}, virtually all the information of the gradient in a convolutional architecture is captured in the first few singular vectors. Gradients here are shown for a single $C \times C$ block in the Fourier regime of the convolution filter parameterized via our orthogonal \textsc{projUNN} convolution (see \cref{sec:projunn_conv} and \cref{app:convolution_unitary} for form of parameterization). This filter is contained in the last residual block of the Resnet-9 network \cite{he2016deep} and has $512$ channels so the gradient is a $512 \times 512$ matrix. Since networks were trained with a batch size of 128, the maximum rank of this gradient is actually 128 (see short proof below in \cref{prop:convnet_rank_bound}).

\begin{figure}
    \centering
    \includegraphics[]{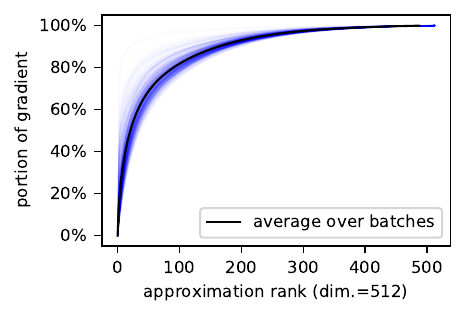}
    \caption{Low rank approximations capture most of the Frobenius norm of the gradient. Here, we plot gradients of the matrix in the RNN's hidden layer for each batch (light blue line) over an epoch of training \textsc{projUNN} in the pixel-by-pixel MNIST task (see \cref{sec:experiments}).
    }
    \label{fig:mnist_low_rank}
\end{figure}

Since the weight matrix in hidden layers of RNNs is repetitively applied, gradients tend to be of a higher stable rank in these settings. Nevertheless, as evident in \cref{fig:mnist_low_rank}, gradients over the course of an epoch of training an RNN  of hidden dimension 512 on the sequential MNIST task still exhibit clear low rank behavior, albeit not to the same degree as convolutional architectures. In constructing these plots, a batch size of 128 was used and networks were trained using the RMSprop algorithm. Throughout our experience with training the \textsc{projUNN} on various tasks, we observe that setting the rank of an approximation to even just one is effective at ensuring convergence of training to a good solution.

We furthermore note that the exact rank of a gradient in neural network architectures can typically be bounded by the dimensions of the input. For convolutional networks, this exact rank is bounded by the batch size and for vanilla recurrent neural networks, this exact rank is bounded by the batch size times the sequence length. We provide formal statements and short proofs of these propositions below.

\begin{proposition} 
\label{prop:RNN_rank_bound}
Given a loss function $\ell:\mathbb{R}\times\mathbb{R} \to \mathbb{R}$ taking in two real numbers and outputting a real number, let $f^{(t)}$ denote the $t$-th sequential output of a vanilla RNN defined as
\begin{equation}
\begin{split}
        \vh^{(t)}(\vx) &=\sigma(\mM\textbf{x}^{(t)}+\mW \vh^{(t-1)}(\vx) ) \\
        f^{(t)}(\vx) &= \tau(\mV \vh^{(t)}(\vx) ),
\end{split}
\end{equation}
where $t \in [T]$ indexes the sequence of inputs, inputs $\vx^{(t)} \in \mathbb{R}^{d}$ ($\vx$ denotes the concatenation of all inputs), input transformation $\mM \in \mathbb{R}^{h \times d}$, hidden transformation $\mW \in \mathbb{R}^{h \times h}$, output transformation $\mV \in \mathbb{R}^{h \times d}$, and $\sigma$ and $\tau$ are pointwise non-linearities. Let $\nabla_{\mW} L$ indicate the gradient of the hidden transformation matrix $\mW$ with respect to the loss function $L = \sum_{i=1}^b \ell(f^{(T)}(\vx_i),y_i) $ over a batch of inputs $\{ \vx_i,y_i\}_{i=1}^b$. The rank of $\nabla_{\mW} L$ is bounded above by $bT$.
\end{proposition}
\begin{proof}
Define $\vz^{(t)} = \mM\textbf{x}^{(t)}+\mW \vh^{(t-1)}(\vx)$. Calculating the gradient of the loss function with respect to the parameter $\mW$, we have:
\begin{equation}
\begin{split}
    \nabla_{\mW} L &= \sum_{i=1}^b \nabla_{\mW} \ell(f^{(T)}(\vx_i),y_i) \\
    &= \sum_{i=1}^b \sum_{t=1}^T \sum_{k=1}^h \left[\frac{\partial \ell(f^{(T)}(\vx_i),y_i)}{\partial \vz^{(t)}} \right]_k \frac{\partial \vz^{(t)}_k}{\partial \mW} \\
    &= \sum_{i=1}^b \sum_{t=1}^T \left[\frac{\partial \ell(f^{(T)}(\vx_i),y_i)}{\partial \vz^{(t)}} \right] \left[\vh^{(t-1)}(\vx_i)\right]^\intercal.
\end{split}
\end{equation}

The above is the sum of $bT$ rank one matrices which is at most rank $bT$ concluding the proof.
\end{proof}

\begin{proposition} 
\label{prop:convnet_rank_bound}
Given a loss function $\ell:\mathbb{R}\times\mathbb{R} \to \mathbb{R}$ taking in two real numbers and outputting a real number and convolution filter $\tW \in \mathbb{R}^{ N \times M \times C \times C }$, let $f(\tX)$ denote the convolutional neural network defined as
\begin{equation}
\begin{split}
        \tH(\tX) &=g_{pre}(\tX) \\
        \tY(\tX) &= \conv_{\tW}(\tH(\tX)) \\
        f(\tX) &= g_{post}(\tY(\tX)),
\end{split}
\end{equation}
where input $\tX \in \mathbb{R}^{ N_{in} \times M_{in} \times C_{in}}$ and $g_{pre}$ and $g_{post}$ are functions which apply the layers before and after the convolution with filter $\tW$. Let $\widehat{\tW}_{\widehat{r},\widehat{s},:,:}$ denote the $C \times C$ matrix storing the values of the convolution filter in the Fourier regime for frequencies $\widehat{r}$ and $\widehat{s}$ (see \cref{eq:app_conv_fourier}). Let $\nabla_{\widehat{\tW}_{\widehat{r},\widehat{s},:,:}} L$ indicate the gradient of $\widehat{\tW}_{\widehat{r},\widehat{s},:,:}$ with respect to the loss function $L = \sum_{i=1}^b \ell(f(\tX_i),y_i) $ over a batch of inputs $\{ \tX_i,y_i\}_{i=1}^b$. Then, for all $\widehat{r}$ and $\widehat{s}$ in the support of the filter, the rank of $\nabla_{\widehat{\tW}_{\widehat{r},\widehat{s},:,:}} L$ is bounded above by $b$.
\end{proposition}
\begin{proof}
We have that
\begin{equation}
\begin{split}
    \nabla_{\widehat{\tW}_{\widehat{r},\widehat{s},:,:}} L &= \sum_{i=1}^b \nabla_{\widehat{\tW}_{\widehat{r},\widehat{s},:,:}} \ell(f(\tX_i),y_i).
\end{split}
\end{equation}
We now proceed to show that the rank of $\nabla_{\widehat{\tW}_{:,:,\widehat{r},\widehat{s}}} \ell(f(\tX_i),y_i)$ is equal to one which completes the proof. As a reminder, we have via the convolution theorem that
\begin{equation}
    \left[\FFT \conv_{\tW}(\tX) \right]_{\widehat{r},\widehat{s},:} = \widehat{\tW}^*_{\widehat{r},\widehat{s},:,:} \; \left[\FFT\tX \right]_{\widehat{r},\widehat{s},:} .
\end{equation}

Let $\widehat{\tY}(\tX) = \FFT \conv_{\tW}(\tX)$, then we have
\begin{equation}
\begin{split}
    \nabla_{\widehat{\tW}_{:,:,\widehat{r},\widehat{s}}} \ell(f(\tX_i),y_i) &= \sum_{k=1}^C \frac{\partial \ell(f(\tX_i),y_i)}{\partial \widehat{\tY}(\tX)_{\widehat{r},\widehat{s},k}} \frac{\partial \widehat{\tY}(\tX)_{\widehat{r},\widehat{s},k}}{\tW_{\widehat{r},\widehat{s},:,:}} \\
    &= \frac{\partial \ell(f(\tX_i),y_i)}{\partial \widehat{\tY}(\tX)_{\widehat{r},\widehat{s},:}} \left( \left[\FFT\tX \right]_{\widehat{r},\widehat{s},:} \right)^\intercal.
\end{split}
\end{equation}
The above is a rank one matrix which concludes the proof.
\end{proof}

\subsection{Other sampling algorithms}
\label{app:other_sampling}
The prior analysis showed that gradients are typically low rank, and in the main text, we listed a couple efficient algorithms that allow one to approximate a full but approximately low rank matrix with an explicitly low rank matrix. Enhancements to these algorithms exist that may provide further improvements in very large dimensions. For example, the \emph{FKV sampling} \cite{frieze2004fast} and \emph{constant time SVD} \cite{drineas2006fast} algorithms provide runtimes that are often logarithmic in the dimension of the matrix for sampling entries from the low rank approximation to a matrix. Nevertheless, explicitly constructing all entries of a rank $k$ approximation to an $n \times n$ matrix $\mA$ requires time at least $O(k^2n)$. Improvements in runtime are achieved by applying random projections to both the column and row subspaces of the matrix $\mA$ to perform the final approximation. In practice, the runtime of these algorithms depends on other factors that limit the applicability of the algorithm even for relatively large matrices. Numerical simulations show that FKV sampling achieves a practical speedup for matrices of dimension approximately $10^6$ or higher \cite{arrazola2019quantum}. Since matrices in our \textsc{projUNN} were of much smaller dimension, we did not use these methods.

\section{Runtime comparisons}
\label{app:runtime}

\cref{tab:model_comps} (see main text) and \cref{tab:conv_model_comps} lists the asymptotic runtime complexities of unitary/orthogonal models in the RNN and convolutional setting respectively. For RNNs, \textsc{projUNN} has the best runtime scaling of all models which parameterize the full orthogonal/unitary space with just one layer. In fact, apart from the rank approximation factor $k$, the runtime of \textsc{projUNN} is optimal in the RNN setting. In the convolutional setting, \textsc{projUNN} scales most efficiently when there are many channels or the filter size $W$ scales faster than the logarithm of the input dimension $\log N$. For small filter sizes, BCOP and SOC scale relatively efficiently. In fact, SOC is nearly optimal for $W \ll \log N$ up to the approximation factor $p$ in the Taylor expansion which is required both upon evaluation and training of the network.

\begin{table*}[!t]
\small
\centering
\begin{tabular}{lcl}
\hline
Model & \begin{tabular}[c]{@{}c@{}}Complexity of\\ forward + backward step\end{tabular} & Notes \\ \hline
BCOP $^1$ & $O(C^2W^2N^2 + C^3W^3)$  & $W$ sequential orthogonalizations needed (slow for large $W$) \\
SOC $^2$ & $O(pC^2W^2N^2)^a$ & $p$ denotes number of terms in Taylor series approximation       \\
Cayley $^3$ & $O(N^2C^2 \log(N) + N^2C^3)^b $ &        \\
\textsc{projUNN} (our method) & $O(N^2C \log(N) + kN^2C^2)^c $ & $k$ denotes rank of gradient updates              \\ \hline
\multicolumn{3}{l}{\begin{tabular}[c]{p{0.95\linewidth}}\small $^1$ \cite{li2019preventing}, $^2$ \cite{singla2021skew}, $^3$ \cite{trockman2021orthogonalizing} \\ $^a$ just applying the convolution (without gradient update) also requires the added factor $p$ in runtime unlike standard convolutions and BCOP which run in time $O(C^2W^2N^2)$ upon just evaluation, $^b$ approximations to matrix inversion exist which may reduce runtimes though these approximations are not implemented here, $^c$ runtime shown for typical setting when $k \ll n$ \end{tabular}}
\end{tabular}
\caption{Time complexity of 2-D orthogonal convolutional layers with filter size $W \times W$ applied to $N \times N$ inputs with $C$ input and output channels.
}
\label{tab:conv_model_comps}
\end{table*}

\paragraph{Runtimes for RNN implementations in practice} Enforcing unitarity in \textsc{projUNN} incurs a computational overhead associated to performing eigendecompositions and QR decompositions entailed in updating the gradient. Even though such computations are performed on a small subspace of the total dimension of the matrix, such computations may not effect increase training times by a constant factor as evident in \cref{fig:runtimes}. Empirically, we observed that much of the increased overhead was due to performing eigendecompositions and QR decompositions on a GPU, both tasks which are challenging to parallelize on GPU architectures. We note that similar issues may be one reason why the expRNN runtimes appear much slower in our simulations as shown in \cref{fig:runtimes}. Recent updates to pytorch \cite{NEURIPS2019_9015} and tensorflow \cite{tensorflow2015-whitepaper} have focused on improving the runtimes of these common linear algebraic routines, and we expect these runtimes to continue to improve in the future. 

\begin{figure}
    \centering
    \includegraphics{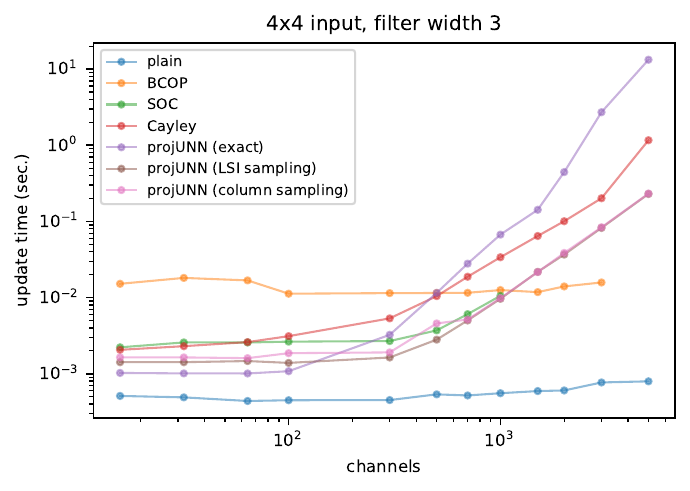}
    \caption{Comparison of runtimes of orthogonal convolutional architectures when varying the number of channels (log-log plot). The number of input channels is equal to the number of output channels. Runtimes are averaged over 10 forward and backward passes through a single convolutional operation on random data. \textsc{projUNN} with LSI sampling and column sampling have very similar runtimes and may appear to completely overlap on the chart. Out of memory error obtained for BCOP for large number of channels.}
    \label{fig:conv_runtime_channels}
\end{figure}

\begin{figure}
    \centering
    \includegraphics{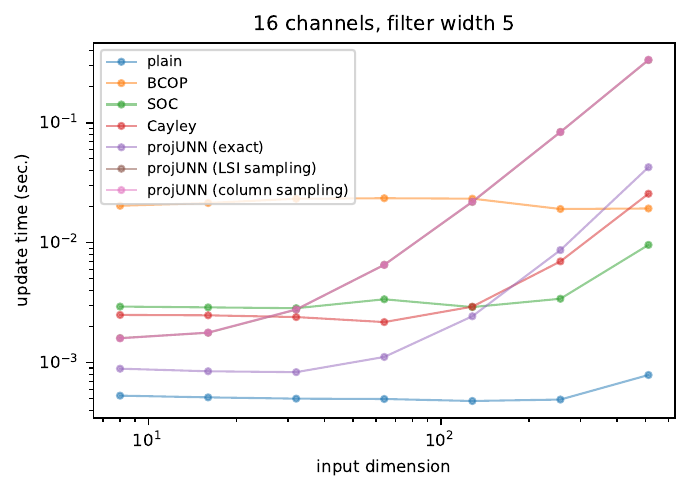}
    \caption{Comparison of runtimes of orthogonal convolutional architectures when varying the size of the input which is a square image (log-log plot). Runtimes are averaged over 10 forward and backward passes through a single convolutional operation on random data. \textsc{projUNN} with LSI sampling and column sampling have very similar runtimes and may appear to completely overlap on the chart. }
    \label{fig:conv_runtime_imsize}
\end{figure}

\begin{figure}
    \centering
    \includegraphics{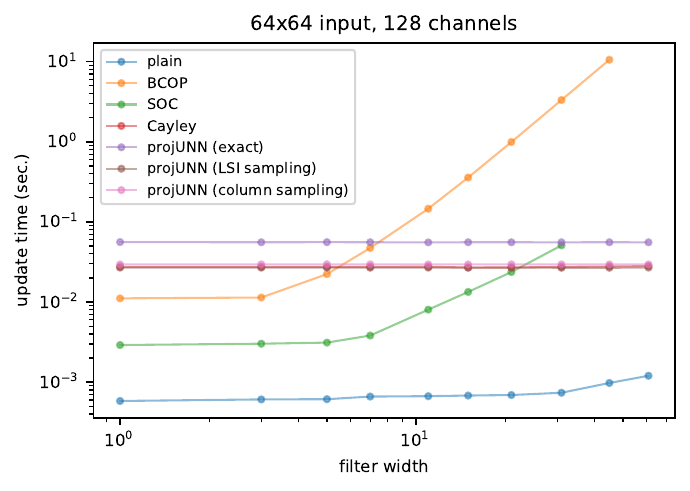}
    \caption{Comparison of runtimes of orthogonal convolutional architectures when varying the size of the filter (log-log plot). Both \textsc{projUNN} and the cayley convolution \cite{trockman2021orthogonalizing} perform operations on the full space of convolution filters so the filter size does not change the runtime for these models. Runtimes are averaged over 10 forward and backward passes through a single convolutional operation on random data. \textsc{projUNN} with LSI sampling and column sampling have very similar runtimes and may appear to completely overlap on the chart. Out of memory error obtained for SOC beyond filter size of 31. }
    \label{fig:conv_runtime_filter}
\end{figure}

\begin{figure}
    \centering
    \includegraphics{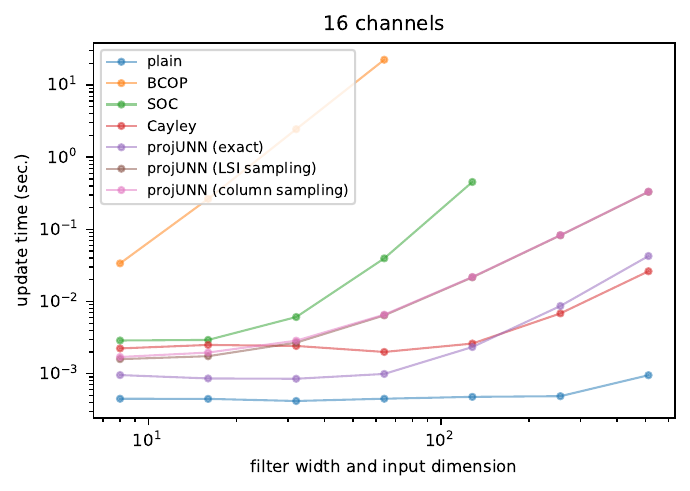}
    \caption{Comparison of runtimes of orthogonal convolutional architectures when varying the size of the filter and the input dimension in-tandem (log-log plot). Runtimes are averaged over 10 forward and backward passes through a single convolutional operation on random data. \textsc{projUNN} with LSI sampling and column sampling have very similar runtimes and may appear to completely overlap on the chart. Out of memory error obtained for SOC beyond filter size and image size of 128. }
    \label{fig:conv_runtime_filter_and_imsize}
\end{figure}

\paragraph{Runtimes for convolutional network implementations in practice} \cref{fig:conv_runtime_channels}, \cref{fig:conv_runtime_imsize}, and \cref{fig:conv_runtime_filter} compare runtimes of orthogonal convolution algorithms when varying the number of channels, input size, and filter size respectively. \cref{fig:conv_runtime_filter_and_imsize} plots the runtime when the filter and image size are set equal to each other and varied together. In summary, we find that the skew orthogonal convolution (SOC) \cite{singla2021skew}, the Cayley model \cite{trockman2021orthogonalizing}, and our \textsc{projUNN} are all relatively efficient for small to medium sized models. \textsc{projUNN} empirically performs best in comparison to other models when there are many channels or the filter size is large.  SOC \cite{singla2021skew} is the fastest when the filter size is small. As a reminder, SOC \cite{singla2021skew} employs the Taylor series approximation to form a unitary/orthogonal matrix. We did not change the number of terms in the approximation used even though increasing the size of a filter, dimension of an input, or the number of channels will likely require more terms to maintain the same error in the approximation. 

The Cayley model \cite{trockman2021orthogonalizing} has a runtime that is empirically similar to \textsc{projUNN} in most of our experiments despite \textsc{projUNN} having an improved asymptotic scaling with regards to the number of channels (see \cref{tab:conv_model_comps}). Runtimes here are calculated by averaging across the first 10 steps of gradient descent on random data for initialized networks. Close to initialization, matrices in the convolution filter are generally well conditioned and this may be one reason why the Cayley model showed good scaling  in \cref{fig:conv_runtime_channels} since the matrix inversion for the well conditioned matrices runs quickly. We also suspect that for the number of channels considered, the added costs of performing sampling and projections in \textsc{projUNN} do not materialize into a very clear runtime benefit. Interestingly, the exact version of \textsc{projUNN} (no low rank approximation) runs in similar time to the Cayley method.

Finally, we note that in comparison to RNN implementations, there is increased overhead associated to performing eigendecompositions and QR decompositions with our \text{projUNN} convolutional network implementations since unitary matrices are batched across the elements of the convolutional filter. As mentioned before, much of the increased overhead was due to performing eigendecompositions and QR decompositions on a GPU, both tasks which can be challenging to parallelize on GPU based architectures. Future updates to Pytorch \cite{NEURIPS2019_9015} and Tensorflow \cite{tensorflow2015-whitepaper} may help improve runtimes of these operations.

\section{Network architectures and training details}
\label{app:network_architectures}

\begin{center}
    \begin{tabular}{c|c|c}
        ExpRNN trivializations &\url{https://github.com/Lezcano/expRNN} & MIT  \\
        scoRNN &\url{https://github.com/SpartinStuff/scoRNN} & NA\\
        scuRNN &\url{https://github.com/Gayan225/scuRNN}& NA\\
        skewed orthogonal convolutions (SOC) &\url{https://github.com/singlasahil14/SOC}& NA\\
        Cayley convolution &\url{https://github.com/locuslab/orthogonal-convolutions}& MIT\\
        BCOP & \url{github.com/ColinQiyangLi/LConvNet}&MIT
    \end{tabular}
\end{center}

\subsection{Handling complex numbers}
\label{app:handling_complex_numbers}
Since matrices in the unitary group are complex valued, care must be taken to ensure that the neural network can handle such complex numbers. In these settings, standard activation functions need to be adapted for complex numbers. It is often advantageous to have a nonlinearity which does not change the phase of its input. For these reasons, the activation function used is typically the modRELU activation shown below:
\begin{equation}
    \begin{split}
        \sigma_{modRELU}(z) & = (|z| + b) \frac{z}{|z|} \; \; \text{if} \; \;  \|z\| + b > 0 \\
        \sigma_{modRELU}(z) & = 0 \; \; \text{if} \; \;  |z| + b \leq 0
    \end{split}
\end{equation}
where $b$ is a bias term (trainable). In calculations of the modulus $|z|$, a small additive constant is often added to avoid stability issues when $|z|$ is small.

Another challenge that arises with performing orthogonal convolution is that the representation of a real-valued convolutional filter in the Fourier domain will be complex-valued. More specifically, the fast Fourier transform of a real signal is Hermitian-symmetric. For example, for a one dimensional vector $\vx \in \mathbb{R}^N$ where $\widehat{\vx} = \FFT \vx \in \mathbb{C}^N$ and $\widehat{\vx}_i = \widehat{\vx}_{-i \mod N}$. Therefore, when initializing weight filters in our orthogonal convolution operations, one must be careful to ensure the Hermitian symmetric property holds. Similarly, when performing convolutions in the Fourier space, care must be taken in converting data types to and from complex and real space. In our implementation, we used the pytorch FFT.RFFT2 and FFT.IRFFT2 commands to implement these operations efficiently \cite{paszke2017automatic}.

\subsection{Optimizers}
For training \textsc{projUNN} architectures, a separate optimizer is used for unitary/orthogonal parameters which must be projected after updates and all other parameters. Learning rates for the unitary/orthogonal parameters are typically set to less than that of the other parameters. In our implementations, we found that a learning rate for the unitary or orthogonal parameters of one-tenth or one-twentieth of that of the other parameters works well in practice. Furthermore, when using optimizers with momentum terms, we constructed a variant of standard optimizers to apply changes to the momentum terms after projecting gradients onto the unitary/orthogonal manifold or tangent space. This optimizer performed well in our experiments though sometimes added instability. Therefore, unless otherwise stated, standard optimizers were used.

\subsection{Numerical stability}
\label{app:numerical_stability}
Though updates via \textsc{projUNN} mathematically guarantee that parameters remain unitary/orthogonal, performing a significant number of sequential updates to a matrix can add numerical errors over time slowly drifting parameters away from the unitary/orthogonal manifold. This is especially true in the case of \textsc{projUNN-D} where updating matrices requires division of eigenvalues in the eigendecompositon of the rank $k$ subspace (see \cref{app:deferred_proofs} and \cref{eq:eigen_projunnd}). In contrast, we empirically find that \textsc{projUNN-T} maintains stability throughout optimization even when requiring many updates. To actively address potential instabilities, one can perform eigendecompositions with higher floating point precision or sporadically project unitary/orthogonal parameters onto the closest unitary/orthogonal matrix via \cref{lemma:DirectProjection}. Though this projection step requires $O(n^3)$ time for an $n \times n$ matrix, performing this projection only every $O(n)$ gradient updates can still preserve efficient runtimes of on average $O(kn^2)$ time per gradient update. Unless otherwise stated, we do not perform any additional steps for maintaining stability.

\subsection{Initialization of unitary/orthogonal parameters}
\label{app:initialization}
Empirically, we found that initializing unitary/orthogonal matrices to be Haar random or close to Haar random does not perform well. This is also an observation noted in prior works \cite{henaff2016recurrent,helfrich2018orthogonal,lezcano2019cheap}. Instead, we used one of two different initialization schemes. The first initializes unitary/orthogonal parameters as the identity matrix. The second initializes unitary/orthogonal matrices using variants of the so-called Henaff initialization \cite{henaff2016recurrent} where $2 \times 2$ diagonal blocks of the matrices are initialized as samples from the below
\begin{equation}
    \exp \left( \begin{bmatrix}
    0 & s_i \\ -s_i & 0
    \end{bmatrix} \right),
\end{equation}
where $s_i$ is sampled uniformly from $[-\pi,\pi]$ \cite{henaff2016recurrent}. In other words, an $n \times n$ matrix $\mU$ is initialized as
\begin{equation}
    \mU = \begin{bmatrix}
    \exp \left( \begin{bmatrix}0 & s_1 \\ -s_1 & 0\end{bmatrix} \right) & \boldsymbol{0} & \cdots & \boldsymbol{0} \\
    \boldsymbol{0} & \exp \left( \begin{bmatrix}0 & s_2 \\ -s_2 &0\end{bmatrix} \right) &  & \boldsymbol{0} \\
    \vdots &  & \ddots  & \vdots   \\
    \boldsymbol{0} & \boldsymbol{0} & \cdots &  \exp \left( \begin{bmatrix}0 & s_{n/2} \\ -s_{n/2} & 0 &\end{bmatrix} \right)
    
    \end{bmatrix}.
\end{equation}
Note, that since we parameterize matrices in the Lie group instead of the Lie algebra, we include the exponential map in the parameterization above. Other variants of the above have been used in prior works. For example, the Cayley initialization samples $s_i=\sqrt{\frac{1- \cos t_i}{1+\cos t_i}}$ where $t_i$ is sampled uniformly from $[0,\pi/2]$ \cite{helfrich2018orthogonal}.

\subsection{Computational details}

We employed the latest version of PyTorch (1.10.1+cu102) with Python3.8 and all the pre-requisite libraries coming along. The hardware leverages Quadro GP100 GPU and Intel(R) Xeon(R) CPU E5-2698 v4 @ 2.20GHz.

\end{document}

%% file: math_commands.tex
\usepackage{amsmath,amsfonts,amsthm,bm}

\def\eqref#1{equation~\ref{#1}}

\def\1{\bm{1}}

\def\va{{\bm{a}}}
\def\vb{{\bm{b}}}

\def\vh{{\bm{h}}}

\def\vu{{\bm{u}}}
\def\vv{{\bm{v}}}

\def\vx{{\bm{x}}}
\def\vy{{\bm{y}}}
\def\vz{{\bm{z}}}

\def\mA{{\bm{A}}}
\def\mB{{\bm{B}}}
\def\mC{{\bm{C}}}

\def\mG{{\bm{G}}}

\def\mI{{\bm{I}}}

\def\mK{{\bm{K}}}

\def\mM{{\bm{M}}}

\def\mQ{{\bm{Q}}}

\def\mS{{\bm{S}}}

\def\mU{{\bm{U}}}
\def\mV{{\bm{V}}}
\def\mW{{\bm{W}}}
\def\mX{{\bm{X}}}
\def\mY{{\bm{Y}}}

\DeclareMathAlphabet{\mathsfit}{\encodingdefault}{\sfdefault}{m}{sl}
\SetMathAlphabet{\mathsfit}{bold}{\encodingdefault}{\sfdefault}{bx}{n}
\newcommand{\tens}[1]{\bm{\mathsfit{#1}}}

\def\tH{{\tens{H}}}

\def\tJ{{\tens{J}}}

\def\tL{{\tens{L}}}

\def\tT{{\tens{T}}}

\def\tW{{\tens{W}}}
\def\tX{{\tens{X}}}
\def\tY{{\tens{Y}}}

\DeclareMathOperator*{\argmin}{arg\,min}

\DeclareMathOperator{\Tr}{Tr}